%% file: main.tex
\documentclass{article}

\input{meta_arxiv}

\title{Human Decision-Making with AI Assistance under Correlated Features}

\author{
  Yanru Guan\\
  Carnegie Mellon University\\
  \texttt{yanrug@andrew.cmu.edu}
  \and
  Naveen Raman\\
  Carnegie Mellon University\\
  \texttt{naveenr@cmu.edu}
  \and
  Fei Fang\\
  Carnegie Mellon University\\
  \texttt{feifang@cmu.edu}
}
\date{}

\begin{document}

\maketitle

\begin{abstract}
    \input{sections_arxiv/abstract}
\end{abstract}

\input{sections_arxiv/introduction}
\input{sections_arxiv/model}
\input{sections_arxiv/stationary}
\input{sections_arxiv/optimal_structure}
\input{sections_arxiv/complexity}
\input{sections_arxiv/algorithm}
\input{sections_arxiv/experiments}
\input{sections_arxiv/related}
\input{sections_arxiv/discussion}

\section*{Acknowledgments}
We thank Jason Hartline for helpful discussions and feedback. This work was supported in part by NSF grant IIS-2046640 (CAREER) and by an NSF GRFP Fellowship. 

\newpage 
\bibliographystyle{unsrtnat}
\bibliography{references} 
\newpage 
\appendix
\input{sections_arxiv/appendix}

\end{document}

%% file: meta_arxiv.tex
\usepackage[utf8]{inputenc}
\usepackage[T1]{fontenc}
\usepackage[margin=1in]{geometry}
\usepackage{graphicx}
\usepackage[table,dvipsnames]{xcolor}
\usepackage{amsmath}
\usepackage{amssymb}
\usepackage{amsfonts}
\usepackage{amsthm}
\usepackage{booktabs}
\usepackage{nicefrac}
\usepackage{microtype}
\usepackage{url}
\usepackage{natbib}
\usepackage{tcolorbox}
\tcbuselibrary{skins, breakable}
\usepackage{algorithm}
\usepackage{algpseudocode}
\usepackage[colorlinks=true,linkcolor=blue,citecolor=blue,urlcolor=blue]{hyperref}

\newtheorem{theorem}{Theorem}[section]
\newtheorem{proposition}[theorem]{Proposition}
\newtheorem{lemma}[theorem]{Lemma}
\newtheorem{corollary}[theorem]{Corollary}

\newtheorem{assumption}[theorem]{Assumption}
\newtheorem{definition}[theorem]{Definition}
\theoremstyle{remark}
\newtheorem{remark}[theorem]{Remark}
\theoremstyle{plain}

%% file: sections_arxiv/abstract.tex
Humans increasingly make decisions with AI assistance; for example, doctors may follow AI-recommended diagnostic tests and base their diagnoses on the results. 
A natural question is which tests should AI recommend to balance short-term decision quality and long-term human learning when different features (e.g., test results) are correlated. 
While prior work establishes that stationary policies that recommend the same tests repeatedly are optimal when features are independent, we prove that feature correlations lead such policies to perform arbitrarily poorly.  
Instead, we prove that any optimal policy must follow an explore-then-commit structure; initially, the AI should offer diverse tests so humans can learn accurate feature coefficients, then the AI should commit to a single set of tests, with exploration length that depends on the degree of feature correlation. 
We prove that computing the optimal policy is NP-hard and derive a dynamic programming-based algorithm that finds the optimal policy for finite horizons. 
We additionally develop an approximation that plans for shorter horizons and appends a stationary suffix, achieving near-optimal performance. 
Our empirical results complement our theory by showing that stronger feature correlation leads to longer exploration phases. 

%% file: sections_arxiv/introduction.tex
\section{Introduction}
Artificial intelligence (AI) algorithms are increasingly used as decision-support tools, shaping both the decisions being made and human learning~\citep{generative_ai_work}. 
For example, AI recommends diagnostic tests in healthcare, which in turn reveals patient features (e.g., test results)~\citep{mamography_screening,medical_recommender}. 
AI recommendations influence human learning by shaping the features humans rely on~\citep{chess_teaching,ai_benefits_healthcare}. 
Therefore, determining AI recommendations requires balancing short-term performance and long-term human learning~\citep{ai_assisted_decisions}. 

While many AI systems impact human learning, the structure of optimal recommendation policies remains underexplored. 
Work in machine teaching tackles the set of data points to learn from, but not which tests to recommend~\citep{machine_teaching}, while work in human-AI collaboration has little focused on human learning thus far. 
Most related is recent work~\citep{ai_assisted_decisions} which shows that stationary policies recommending a fixed set of tests over time are optimal under feature independence.
However, test results in healthcare are correlated~\citep{covid_tests}, as are different aspects of defendant's record in criminal justice~\citep{compas}, making the assumption unrealistic.  
Moreover, independent features are a load-bearing assumption; we prove that stationary policies can perform arbitrarily poorly even when features are only weakly correlated. 
Correlated features complicate test selection because recommending one test implicitly reveals information about other related tests.   

In this work, we prove that the optimal policy follows an explore-then-commit strategy.
Without this exploration phase, humans never directly observe certain features and cannot learn accurate coefficient for them, leading to systematically worse decisions even when those features are accounted for indirectly.
Following the initial exploration phase, the optimal policy commits to a stationary set of tests; at this point, humans have sufficiently learned, and it is optimal to offer the most decision-relevant features given the humans' learned coefficients. 
Unlike bandit settings where exploration is used to learn rewards, here exploration is required to shape human coefficient estimates under partial observability.
We prove that exploration length depends primarily on the level of correlation between features; greater correlation leads to longer exploration because there is more value in ensuring humans have accurate coefficients for unobserved features. 

As an illustration of why such results are important, consider a doctor learning to diagnose pneumonia. 
If an AI always recommends chest X-rays and blood cultures, the doctor never directly observes oxygen saturation readings and so never learns its predictive coefficient. 
Later, when the doctor accounts for oxygen saturation implicitly, they apply a poor estimate of its importance, leading to worse diagnoses than if they had been shown the test directly during training.
Our result shows that any AI that does not recommend sufficiently diverse tests for human learning can perform arbitrarily worse than one that initially varies its recommendations.

Our contributions are: 
\begin{enumerate}
    \item Prove that stationary policies perform arbitrarily poorly when features are correlated.
    \item Characterize the optimal policy as explore-then-commit, with the length of the exploration phase determined by the level of feature correlations.
    \item Prove that the general problem is NP-hard and derive a dynamic programming solution that is exact for finite horizons and near-optimal for infinite horizons with explicit error bounds.
    \item Empirically validate the theoretical relationship between feature correlation and exploration\footnote{Code here: \url{https://github.com/YanruGuan/human-ai-collaboration-experiments}}.
\end{enumerate}

%% file: sections_arxiv/model.tex
\section{Model and Problem Setup}
\label{sec:model}
We model a setting where an AI teaches a human through repeated interaction, and is motivated by real-world studies demonstrating human learning through repeated AI interaction~\citep{when_to_advise}

\subsection{Model}
\label{sec:human_learning_model}
The AI teaches a linear function mapping covariates $\mathbf{x} \in \mathbb{R}^{n}$ to labels $y \in \mathbb{R}$. 
At each timestep $t$, we sample $\mathbf{x}^{(t)} \sim \mathcal{X}$ and $\varepsilon^{(t)} \sim \mathcal{E}$ independently, and the label is generated as
\begin{equation}
\label{eq:linear-label-model}
    y^{(t)} = \sum_{i=1}^n a_i x_i^{(t)} + \varepsilon^{(t)}.
\end{equation}
Without loss of generality, assume $\mathcal{X}$ and $\mathcal{E}$ are zero-mean, as we can center both distributions. 

At each timestep $t$, the AI selects a subset of at most $k$ tests $S_t \subseteq [n], |S_{t}| \leq k$, and the human observes the features $\mathbf{x}_{S_t}^{(t)}$.
For example, in a medical decision-making task, each timestep corresponds to a patient showing up, $S_{t}$ refers to the set of medical tests suggested by the AI assistant to be run, $k$ is an upper bound on the number of tests that can be run on a single patient, $\mathbf{x}$ is the results of those tests, and $y$ to the presence of some comorbidity. 
The human makes predictions according to their estimated coefficients $\hat{\mathbf{a}}^{(t)}$, by first imputing missing covariates
\[
    \hat{\mathbf{x}}^{(t)}
    :=
    \mathbb{E}\!\left[\mathbf{x}^{(t)} \mid \mathbf{x}_{S_t}^{(t)}\right],
\]
where the selected coordinates satisfy $\hat{\mathbf{x}}_{S_t}^{(t)}=\mathbf{x}_{S_t}^{(t)}$. The human then predicts the label as 
\begin{equation}
    \label{eq:correlation}
    \hat y^{(t)}
    =
    \sum_{i=1}^n \hat a_i^{(t)} \hat{x}_i^{(t)}.
\end{equation}

We assume that the human and AI both have access to an imputation oracle $\mathbb{E}\!\left[\mathbf{x}^{(t)} \mid \mathbf{x}_{S_t}^{(t)}\right]$.
For example, when covariates are Gaussian, the imputation oracle is a linear map based on $\Sigma$. 
Following~\citet{ai_assisted_decisions}, we model humans as updating their feature coefficient estimates after observing the true label $y^{(t)}$; if feature $i$ has been shown $m_i(t)$ times up to timestep $t$, then
\begin{equation}
\label{eq:learning-rule}
    \bigl(a_i - \hat{a}_i^{(t+1)}\bigr)^2 = \phi\bigl(m_i(t)\bigr)\,\bigl(a_i - \hat{a}_i^{(0)}\bigr)^2.
\end{equation}

Here, $\phi:\mathbb{N}\to[0,1]$ is a non-increasing learning function with $\phi(0)=1$ and $\lim_{m\to\infty}\phi(m)=0$, capturing the intuition that repeated direct observation allows the human to eventually learn the true coefficient. 
We assume the human's coefficient estimates update only upon direct observation of a feature, i.e., when $i \in S_t.$
While imputed values $\hat{x}_{i}$ may provide indirect signal about $a_i$, modeling this secondary learning effect would require additional assumptions about how humans process imputed versus observed features; we leave this extension to future work. 

We seek a policy $\pi(t,\mathbf{m}): \mathbb{N} \times \mathbb{N}^{n} \rightarrow 2^{[n]}$ that maps a current timestep $t$ and the number of times each feature is shown $\mathbf{m}$ to a subset $S_{t}$ with $|S_{t}| \leq k$. 
We minimize the discounted expected loss
\begin{equation}
    \label{eq:objective}
J(\delta,\pi) := \sum_{t=0}^{\infty} \delta^t\,
\mathcal L(\pi(t,\mathbf{m}^{(t)}),\hat a^{(t)}),
\end{equation}
where $\delta \in [0,1)$ is the discount factor and
\[
    \mathcal L(S,\hat a)
    :=
    \mathbb E\!\left[\ell(\hat y,y)\mid S,\hat a\right],
\]
for some $\ell:\mathbb R\times\mathbb R\to\mathbb R_{\ge 0}$.
We relax assumptions on perfect knowledge of $\hat{\mathbf{a}}$, $\phi$, and the imputation oracle in Appendix~\ref{sec:misspecification}, and discuss extensions to non-linear functions in Section~\ref{sec:discussion}. 

\subsection{Warm Up: Test Selection with Two Features}
\label{sec:illustrative-example}
To understand how correlated features impact test selection, we present an illustrative example with $n=2$ features and a $k=1$ testing budget.
Let $\mathcal{X}=\mathcal{N}(0,\Sigma)$, $\ell(\hat{y},y) = \left(y-\hat{y}\right)^2$, $a_{1} = a_{2} = a > 0$, and $\hat{a}_{1}^{(0)} = \hat{a}_{2}^{(0)} = 0$, with the update rule $\phi(m) = \alpha^{-2m}$. 
We explore how the feature correlation, $\rho$, impacts the optimal policy under the covariance matrix
$$
\Sigma
=
\begin{pmatrix}
1 & \rho\\
\rho & 1
\end{pmatrix}.
$$
The AI chooses either $S_t=\{1\}$ or $S_t=\{2\}$, with humans imputing the missing feature 
$$
\mathbb E[x_2^{(t)}\mid x_1^{(t)}]=\rho x_1^{(t)}.
$$
Hence, by Equation~\ref{eq:correlation}, the human predicts
$$
\hat y^{(t)}=(\hat{a}_1^{(t)}+\rho \hat{a}_2^{(t)})x_1^{(t)}.
$$
Similarly, if $S_t=\{2\}$, then
\[
\hat y^{(t)}=(\hat{a}_2^{(t)}+\rho \hat{a}_1^{(t)})x_2^{(t)}.
\]

In the absence of correlations ($\rho=0$), we recover the independent-feature setting where an optimal stationary policy exists ~\citep{ai_assisted_decisions}. 
However, once correlation is introduced ($\rho > 0$), stationary policies can perform poorly because it becomes advantageous to vary which features are shown. 

Formally, let $\Pi_{\mathrm{stat}}$ be the set of stationary policies (i.e., $\pi(t,\mathbf{m})$ is constant in $t$ and $\mathbf{m}$), and let $\Pi_{\mathrm{alt}}$ be the set of policies that alternate between $1$ and $2$ (i.e., $\pi(t,\mathbf{m})$ depending only on the parity of $t$).
Then the optimal policy is stationary only when feature correlations are below a threshold. 
\begin{proposition}[Exact solution in the symmetric two-feature case with geometric learning]
\label{thm:symmetric-two-feature-geometric}
Suppose $n=2$, $k=1$, and $\mathcal{X} = \mathcal{N}(0,\Sigma)$ with
\[
\Sigma=
\begin{pmatrix}
1 & \rho\\
\rho & 1
\end{pmatrix},
\qquad 0\le \rho<1.
\]
Assume $a_1 = a_2 = a$, $\hat{a}_1^{(0)} = \hat{a}_2^{(0)} = 0$, and $\phi(m) = \alpha^{-2m}$ for some $\alpha > 1$.
Define
\[
\rho^{\star}(\delta) = \sqrt{\frac{1-\delta}{1-\delta/\alpha^2}}.
\]
Then the optimal policy is
\[
\pi^{\star} \in
\begin{cases}
\Pi_{\mathrm{stat}}, & \rho < \rho^{\star}(\delta),\\
\Pi_{\mathrm{alt}}, & \rho > \rho^{\star}(\delta),\\
\Pi_{\mathrm{stat}} \cup \Pi_{\mathrm{alt}}, & \rho = \rho^{\star}(\delta).
\end{cases}
\]
\end{proposition}

The optimality of stationary policies under uncorrelated features, as shown by~\citet{ai_assisted_decisions}, breaks when expanded to correlated features. 
The proof constructs explicit value functions for each policy class and verifies that they satisfy the Bellman optimality conditions (full proof in Appendix~\ref{app:symmetric-two-feature-proof}). 
We next show that stationary policies can perform arbitrarily poorly in general. 

%% file: sections_arxiv/stationary.tex
\section{Stationary Policies Can Perform Arbitrarily Poorly with Correlated Features}
We demonstrate that stationary policies, which were found to be optimal by prior work~\citep{ai_assisted_decisions}, can perform arbitrarily poorly once feature correlations are introduced. 
Stationary policies show the same subset of tests, leaving coefficients for unobserved features unlearned. 

\begin{theorem}[Stationary policies can perform arbitrarily poorly]
\label{thm:stationary_bad}
Suppose $1\le k<n$. Let $\ell(\hat{y},y)=\left(y-\hat{y}\right)^2$, $\mathcal{X} = \mathcal{N}(0,\Sigma)$, and \(\Sigma\) be positive definite. 
Let \(\Pi\) be the set of all policies, and let \(\Pi_{\rm stat}\subseteq\Pi\) be the set of stationary policies.
Let $J(\delta,\pi,\phi)$ be the objective from Equation~\ref{eq:objective} with learning function $\phi$. 
If $\Sigma$ is not diagonal, then there exists a constant \(c_{\Sigma,k}>0\) and for every discount factor \(\delta\in(0,1)\), a learning function $\phi_{\delta}$ such that 
\[
    \frac{\inf_{\pi\in\Pi_{\rm stat}} J(\delta,\pi,\phi_{\delta})}
         {\inf_{\pi\in\Pi} J(\delta,\pi,\phi_{\delta})}
    \;\ge\;
    c_{\Sigma,k}\frac{\delta}{1-\delta}.
\]
Consequently, as \(\delta\to 1\), the approximation ratio of the best stationary policy is unbounded.
\end{theorem}
We construct a scenario where one feature has a badly misspecified initial coefficient, and a single early observation corrects it permanently, while any stationary policy must either live with the misspecification forever or permanently sacrifice the better long-run test set. 
As a result, stationary policies cause a persistent \(\Omega(\delta/(1-\delta))\) tail gap; see Appendix~\ref{app:negative-results} for a full proof.

%% file: sections_arxiv/optimal_structure.tex
\section{Optimal Solution is Explore-then-Commit}
\label{sec:optimal_structure}
We prove that the optimal solution consists of two parts: a ``dynamic'' phase where it selects different subsets of tests, then an infinite ``stationary'' phase, where it selects a fixed subset.
The dynamic phase ensures humans learn accurate coefficients for all features and lasts for $T_d$ timesteps. 

\begin{assumption}
\label{ass:regular-loss}
\(\mathcal L(S,\hat{\mathbf{a}}) = \mathbb{E}[\ell(\hat{y},y) | S,\hat{\mathbf{a}}]\) is continuous in \(\hat{\mathbf a}\) and uniformly Lipschitz: there exists \(L_\ell<\infty\) such that, for all feasible \(S\) and all \(\hat{\mathbf a},\hat{\mathbf a}'\),
\[
\left|\mathcal L(S,\hat{\mathbf a})-\mathcal L(S,\hat{\mathbf a}')\right|
\le L_\ell\left\|\hat{\mathbf a}-\hat{\mathbf a}'\right\|_1 .
\]
\end{assumption}

Equation~\ref{eq:learning-rule} specifies the magnitude of the coefficient error but not its sign.
To make the policy state depend only on exposure counts, throughout this section and the algorithmic results we use the deterministic sign-preserving branch
\begin{equation}
\label{eq:count-based-learning}
    \hat a_i(m_i)
    :=
    a_i-(a_i-\hat a_i^{(0)})\sqrt{\phi(m_i)} .
\end{equation}
Thus an exposure-count vector $\mathbf m$ determines the coefficient estimate $\hat a(\mathbf m)$.
This is the natural monotone trajectory consistent with Equation~\ref{eq:learning-rule}; the arguments below only require a deterministic count-based trajectory converging to $a_i$, so another such branch could be substituted consistently.

\begin{theorem}[Eventually stationary suffix]
\label{thm:eventually-static-suffix-regular}
Let \(\pi^\star\) be an optimal policy and $S_{t}=\pi^{\star}(t,\mathbf{m}^{(t)})$. 
Let $R := \{i\in[n]: i\in S_t \text{ for infinitely many }t\}$ be the set of features shown infinitely often, and let $\bar{R} := [n]\setminus R$, with $c_{i}$ being the number of times feature \(i\) is shown for $i \in \bar{R}$. 

Define the policy-specific limiting coefficient vector as
\[
    \bar{a}(\mathbf{a},\mathbf{c})_{i}
    :=
    \begin{cases}
    a_i, & i\in R,\\
    \hat a_i(c_i), & i\in \bar{R}.
    \end{cases}
\]
Here, $\hat a_i(c_i)$ is the frozen coefficient estimate after feature $i$ is shown $c_i$ times. 
If Assumption~\ref{ass:regular-loss} holds, then there exists \(T_d<\infty\) such that
\[
    S_t \in
    \arg\min_{\substack{S\subseteq R\\ |S|\le k}}
    \mathcal L(S,\bar{a}(\mathbf{a},\mathbf{c}))
    \qquad
    \text{for all }t\ge T_d .
\]
Moreover, if \(\mathcal L(S,\bar{a}(\mathbf{a},\mathbf{c}))\) has a unique minimizer
\(S^\star\), then for all $t\ge T_d$, $S_t=S^\star$.
\end{theorem}

Because $\mathcal{L}(S,\hat{\mathbf{a}}^{(t)})$ converges to $\mathcal{L}(S,\bar{a}(\mathbf{a},\mathbf{c}))$, the gain from switching to the optimal tail action eventually dominates the effect on feature learning, contradicting optimality of any non-stationary late-round action (see Appendix~\ref{app:eventual-static-suffix}). 
Critically, standard explore-then-commit proofs fix an exploration length and bound the regret of a specific algorithm, whereas our proof works backwards from optimality and demonstrates the unprofitability of deviations. 
This proof gives a sufficient late-time condition for stationarity, not a tight characterization of the transition time. In particular, it does not rule out the stationary suffix beginning much earlier; matching lower bounds on the minimal such \(T_d\) would require constructing an optimal policy whose nonstationary prefix is provably necessary, so we leave such a characterization open.

We note that $\mathcal L(S,\bar{a}(\mathbf{a},\mathbf{c}))$ having multiple minimizers is essentially a degenerate case; in many common distributions and losses, only a measure-zero set of coefficients $a$ lead to multiple minimizers for $\mathcal L(S,\bar{a}(\mathbf{a},\mathbf{c}))$ (see Corollary~\ref{cor:generic-eventual-stationarity-regular}).  

%% file: sections_arxiv/complexity.tex
\section{Computational Complexity}
\label{sec:complexity}
We demonstrate that determining optimal tests when features are correlated is computationally intractable, with no FPTAS. 

\begin{theorem}[NP-hardness of test selection]
\label{thm:np_hard}
The decision version of the test selection problem is NP-hard. 
Moreover, unless \(P=NP\), the corresponding minimization problem admits no FPTAS. 
\end{theorem}

The proof reduces from vertex cover on cubic graphs~\citep{vertex_cover_hard} by encoding graph adjacency in the covariance matrix:
covering edges corresponds to selecting tests that drive the
conditional-covariance loss below a threshold (see  Appendix~\ref{app:hardness-proof} for a full proof).
The lack of an FPTAS eliminates any hope of finding polynomial time algorithms that are near-optimal. 

%% file: sections_arxiv/algorithm.tex
\section{Optimal and Approximate Solutions}
Motivated by the computational hardness of optimal testing, we show that a dynamic programming-based solution is optimal for the finite-horizon problem, then develop an approximation. 

\subsection{An Optimal Finite-Horizon Dynamic Programming Solution}
\label{sec:finite_horizon_dp}

Throughout this subsection, $T$ denotes the finite planning horizon, so the policy chooses actions for rounds $t=0,\dots,T-1$.
Let $\hat a(\mathbf m)$ be the count-based coefficient estimate from Equation~\ref{eq:count-based-learning}, and let $\mathbf{1}_{S}$ denote the 0-1 vector with 1s for indices in $S$.
Then we can write the total loss recursively using a value function $V_{t}(\mathbf{m})$ that selects the optimal tests at each timestep.
\begin{align}
V_0(\mathbf{0}) &= 0, \quad V_0(\mathbf{m} \neq \mathbf{0}) = +\infty, \\
V_{t+1}(\mathbf{m})
&=
\min_{\substack{S \subseteq [n],\, |S|\le k \\ \mathbf{m} - \mathbf{1}_S \ge 0}}
\left\{
V_t(\mathbf{m} - \mathbf{1}_S)
+
    \delta^t \mathcal L\bigl(S,\hat a(\mathbf{m} - \mathbf{1}_S)\bigr)
\right\}.
\end{align}

\begin{proposition}[DP Optimality]
\label{prop:finite-horizon-dp}
For any $t\leq T$, \(V_t(\mathbf m)\) is the exact minimum total loss among all length-\(t\)
action sequences that offered each test $m_{i}$ times.
Consequently, \(\min\limits_{\mathbf m}V_T(\mathbf m)\) is the exact optimal finite-horizon
loss.
Assuming that evaluating $\mathcal L(S,\hat{a}(\mathbf m))$ takes $O(1)$ time, the time complexity for computing the optimal policy is $O\!\left(T^n \sum_{j=0}^k \binom{n}{j}\right)$ and the memory
complexity is $O(T^n)$.
\end{proposition}

After filling the table for the finite horizon $T$, we choose a terminal state minimizing \(V_T(\mathbf m)\) and backtrack to recover the optimal policy.
Each test can be selected at most $T$ times, and there are $n$ tests, yielding a memoization array of size $O(T^{n})$.
For each count state, we consider $\sum_{j=0}^k\binom{n}{j}$ feasible subsets, yielding a total runtime of $O\!\left(T^n \sum_{j=0}^k\binom{n}{j}\right)$ (proof in Appendix~\ref{app:finite-dp-proof}).

\begin{algorithm*}[t]
  \caption{Optimal Dynamic Programming Solution for Finite Horizons}
  \label{algo:dp}
  \begin{algorithmic}[1]
  \Require $\mathcal{L}$, $a$, $\phi$, horizon $T$, budget $k$, and discount $\delta$
  \State $V_0(\mathbf 0) \leftarrow 0$; $V_0(\mathbf m)\leftarrow +\infty$ for all $\mathbf m\neq \mathbf 0$
  \For{$t = 0, 1, \dots, T-1$}
    \State $V_{t+1}(\mathbf m)\leftarrow +\infty$ for all $\mathbf m \in \{0,\dots,T\}^n$
    \For{$\mathbf m$ with $V_t(\mathbf m)<+\infty$}
      \For{$S \subseteq [n]$ with $|S| \le k$}
        \State $c \leftarrow V_t(\mathbf m)+\delta^t \mathcal L(S,\hat a(\mathbf m))$
        \If{$c < V_{t+1}(\mathbf m+\mathbf 1_S )$}
          \State $V_{t+1}(\mathbf m+\mathbf 1_S)\leftarrow c$, $P_{t+1}( \mathbf m+\mathbf 1_S )\leftarrow(\mathbf m,S)$
        \EndIf
      \EndFor
    \EndFor
  \EndFor
  \State Choose $\mathbf m_T^\star\in\arg\min_{\mathbf m}V_T(\mathbf m)$
  \State Backtrack through $P_T,\dots,P_1$ to recover $S_0^\star,\dots,S_{T-1}^\star$
  \State Return $S_0^\star,\dots,S_{T-1}^\star$
  \end{algorithmic}
\end{algorithm*}

\subsection{Approximating Infinite Horizon with Truncated Finite-Horizon DP}
\label{sec:truncated_dp}
The finite-horizon DP from Section~\ref{sec:finite_horizon_dp} can approximate the discounted infinite-horizon problem.
For the infinite-horizon objective, we choose a truncation length \(\bar{T}\) and run the finite-horizon DP for only the first \(\bar{T}\) rounds.
To obtain an infinite-horizon policy, this prefix can be followed by any suffix of tests. 
We prove that, no matter how this suffix is chosen, we can bound performance based on $\bar{T}$. 
\begin{proposition}[Truncated-horizon DP gives an additive approximation]
\label{prop:cutoff-additive}
Define
\[
    J(\delta,\pi)
    :=
    \sum_{t=0}^\infty \delta^t \mathcal L(S_t,\hat a(\mathbf m^{(t)})),
    \qquad
    J_{\bar{T}}(\delta,\pi)
    :=
    \sum_{t=0}^{\bar{T}-1} \delta^t \mathcal L(S_t,\hat a(\mathbf m^{(t)})).
\]
Let $\pi^\star$ be the optimal policy for the infinite-horizon objective $J(\delta,\pi)$, and let $\pi^{(\bar{T})}$ be any infinite-horizon policy whose first $\bar{T}$ actions are optimal for the finite-horizon objective
$J_{\bar{T}}(\delta,\pi)$.
If $\mathcal{L}(S,\hat{a}(\mathbf{m})) \leq \bar{L}$, then
\[
    J(\delta,\pi^{(\bar{T})}) - J(\delta,\pi^\star)
    \le
    \frac{\bar{L}\,\delta^{\bar{T}}}{1-\delta}.
\]
Consequently, additive error is at most $\varepsilon>0$ once
\[
    \bar{T}
    \ge
    \left\lceil
    \frac{\log\!\bigl(\bar{L}/((1-\delta)\varepsilon)\bigr)}
    {-\log\delta}
    \right\rceil.
\]
\end{proposition}

The dependence on the suffix-agnostic tail term in Proposition~\ref{prop:cutoff-additive} is tight in the worst case; see Remark~\ref{rem:cutoff-tail-tightness}.

In Appendix~\ref{app:truncation-proof}, we prove this by bounding $J(\delta,\pi)$ term-wise, and noting that $\pi^{(\bar{T})}$, which is the policy computed by the finite-horizon DP algorithm, is optimal for the first $\bar{T}$ rounds.
In practice, we can compute the stationary completion by repeating the last selected set of tests infinitely. 
Because the DP algorithm proceeds forward in time, we can run the algorithm until test selections converge.

%% file: sections_arxiv/experiments.tex
\section{Experiments}
\label{sec:experiments}
We empirically investigate how different policies affect test selection and performance using synthetic instances that isolate the impact of different parameters.
Unless varied explicitly, we fix $\ell(\hat{y},y) = \left(y-\hat{y} \right)^2$, $\delta=0.99$, $\mathcal{X} = \mathcal{N}(0,\Sigma)$, and $\mathcal{E} = \mathcal{N}(0,\sigma_{\varepsilon}^2)$ with $\sigma_\varepsilon^2=10^{-3}$.
We report results over 20 random instances of $a_{i} \sim \mathrm{Uniform}[0,1]$ and $\hat{a}_{i} \sim \mathrm{Uniform}[0,1]$, together with 95\% confidence intervals. 
Appendix~\ref{sec:beam} explores policy performance for larger values of $n$ and $k$.
Appendix~\ref{sec:misspecification} studies misspecification in $\hat{\mathbf{a}}$, $\Sigma$, the imputation oracle, and $\phi$, showing that perturbations of up to 50\% reduce performance by less than 20\%.
Appendix~\ref{sec:model-variant-experiments} reports the variants of \(\mathcal X\), \(\mathcal L\), and \(\phi\).
That is, our results are robust to parameter misspecification and persist across the tested changes to the feature distribution, loss function, and human learning curve.
We report \emph{retained optimality} as the inverse loss ratio:
\[
    \mathrm{Retained}(\pi)
    :=
    \frac{J(\delta,\pi^\star)}{J(\delta,\pi)}.
\]
Thus $100\%$ means that a candidate policy matches the finite-horizon optimal policy, while smaller percentages measure the loss from restricting the policy class or truncating the planning horizon.
All experiments were run on 16 GB of RAM and completed within 2 hours. 

\subsection{Stationary Policy Performance}
We compare the best stationary policy against the optimal policy.
For $n=2$ and $k=1$, define 
\[
\Sigma(\rho)=
\begin{pmatrix}
1 & \rho\\
\rho & 1
\end{pmatrix},
\qquad 0\le \rho \le 0.99,
\]
We vary $\rho$, and for each $\rho$, we either vary $\alpha$ in the human learning function  $\phi(m)=\alpha^{-2m}$ with $\delta=0.99$ or the discount factor $\delta \in \{0.85,0.90,0.95,0.99\}$ while holding $\alpha=1.10$.

\begin{figure}
    \centering
    \includegraphics[width=\linewidth]{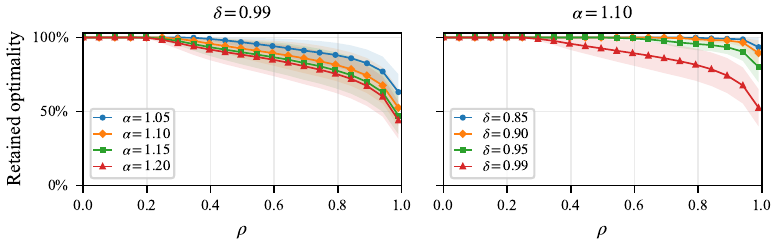}
    \caption{Retained optimality of the best stationary policy relative to the finite-horizon DP algorithm for $n=2,k=1$. For both, we vary the feature correlation $\rho$ when $\mathcal{X} = \mathcal{N}(0,\Sigma)$, while on the left we vary $\alpha$ in the human learning function $\phi(m)=\alpha^{-2m}$, and on the right we vary the discount factor $\delta$. As features are more correlated, stationary policies perform worse, with worse performance as human learning is faster or the discount factor is larger.}
    \label{fig:stationary-policy-retained-optimality}
\end{figure}

Figure~\ref{fig:stationary-policy-retained-optimality} shows that stationary policies are optimal when features are weakly correlated, but perform worse as correlation increases. This gives a finite-horizon empirical counterpart to Theorem~\ref{thm:stationary_bad}: correlation creates value for a dynamic prefix that shapes the human's estimates before exploiting a fixed suffix.
For example,increasing $\rho$ from $0.5$ to $0.99$ reduces performance from 93\% to 53\%. 
Stationary policies perform worse as humans learn faster (larger $\alpha$) or when the discount factor ($\delta$) is larger. For example, increasing $\alpha$ from $1.05$ to $1.20$ at $\rho=0.5$ reduces the retained performance from 97\% to 89\% of optimal. 

\begin{figure}
    \centering
    \includegraphics[width=\linewidth]{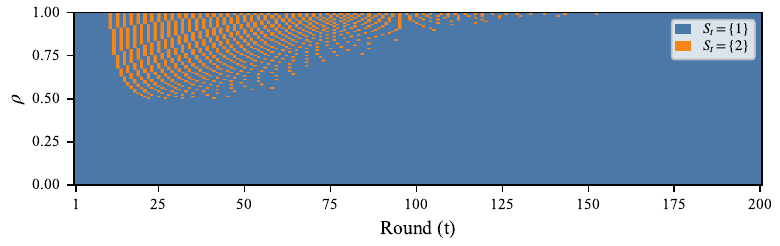}
    \caption{Finite-horizon DP algorithm for one fixed $(n,k)=(2,1)$ instance with $\alpha=1.05$ and $\delta=0.99$. Each row fixes a correlation level $\rho$, each column is a round $t$, and the color indicates whether the DP algorithm selects $S_t=\{1\}$ or $S_t=\{2\}$. Larger $\rho$ leads to a longer exploratory prefix before the policy settles on its stationary suffix action.}
    \label{fig:optimal-action-map-by-alpha}
\end{figure}

To better understand \textit{why} the DP algorithm outperforms stationary policies, we visualize what tests the optimal DP algorithm selects for one instance in 
Figure~\ref{fig:optimal-action-map-by-alpha}.
The DP algorithm matches stationary algorithms when $\rho \leq 0.5$, while increases in feature correlation lead to longer exploration periods. 
For example, for $\rho=0.6$, the DP algorithm explores until $T_d = 49$, while $\rho=0.99$ explores past $T_d>150$. 

\subsection{Dynamic Exploration Length}
\label{sec:dynamic_phase_length}
As shown in Figure~\ref{fig:optimal-action-map-by-alpha}, the dynamic exploration length \(T_d\) varies based on the level of correlation.
Concretely, we compare the theorem-based sufficient certificate for \(T_d\), from the proof of Theorem~\ref{thm:eventually-static-suffix-regular}, with the number found empirically.
We vary $\rho$ for settings $n=2,k=1$ and $n=3,k=2$.
For each random draw, \(T_d\) is an integer: we measure it by solving finite-horizon DP at the configured horizon for each setting and scanning backward to find the last round at which the policy switches actions.

\begin{figure*}[t]
    \centering
    \includegraphics[width=\linewidth]{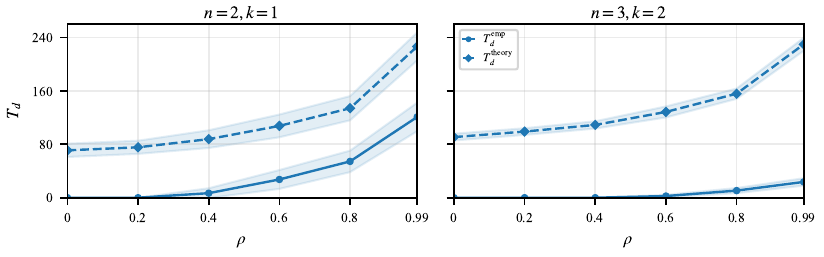}
    \caption{Dynamic phase length $T_d$ versus correlation $\rho$ for $\alpha=1.05$ and $\delta=0.99$, for $n=2,k=1$ and $n=3,k=2$. Curves show sample means over 20 random draws of $(a,\hat a^{(0)})$: solid curves are empirical integer suffix-start times from finite-horizon DP, and dashed curves are conservative sufficient certificates from Theorem~\ref{thm:eventually-static-suffix-regular}. As feature correlations increase, dynamic phase lengths also increase, though at slower rates empirically than theoretically predicted.}
    \label{fig:static-suffix-thresholds}
\end{figure*}

Figure~\ref{fig:static-suffix-thresholds} demonstrates that $\rho$ increases dynamic phase lengths $T_{d}$ both empirically and theoretically. 
For both choices of $(n,k)$, we find that the theoretical bound is much larger than the one found in practice because we prove a conservative upper bound.
Moving from $(n,k)=(2,1)$ to $(3,2)$ substantially shortens the dynamic phase: across the plotted $\rho$ values, the sample mean of the integer-valued empirical \(T_d\) is at most $23.65$ for $n=3,k=2$, while it reaches $121.00$ for $n=2,k=1$ at $\rho=0.99$.
This occurs because when $k=2$, each test set selection reveals more information, and the value of exploration decreases.

\subsection{Truncated DP Performance and Runtime}
While the DP algorithm is optimal for finite horizons, it runs slowly, preventing its use for large values of $n$. 
We show that the truncated DP algorithm can reduce runtime without sacrificing performance too much. 
We vary $\bar{T}$ and compute the retained optimality compared to the DP algorithm when $n=2$ and $k=1$ (we run larger values of $n$ in Appendix~\ref{sec:beam}). 

\begin{figure}
    \centering
    \includegraphics[width=\linewidth]{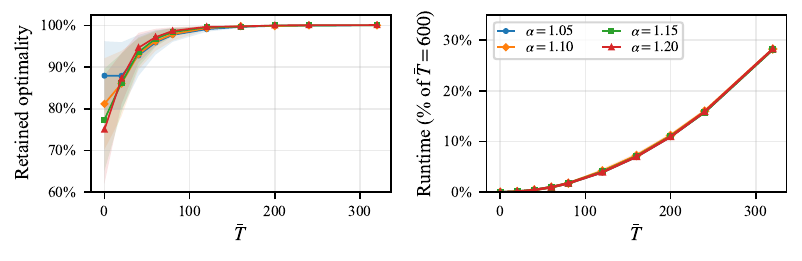}
    \caption{Truncated DP performance-runtime tradeoff for $n=2,k=1$, while fixing $\rho=0.8$, and $\delta=0.99$. 
    We vary $\bar{T}$ and compare performance against the optimal policy (left) and runtime normalized by the full $\bar{T}=600$ DP runtime (right). At $\bar{T} = 120$, truncated DP retains at least $99\%$ of optimality across all curves while using $4\%$ of the $\bar{T}=600$ runtime.}
    \label{fig:truncated-dp-runtime-quality}
\end{figure}

Figure~\ref{fig:truncated-dp-runtime-quality} demonstrates that we can recover much of the performance with small $\bar{T}$. 
At $\bar{T}=80$ for $\alpha=1.10$, we recover $98\%$ of the optimal performance, while running for only $2\%$ of the full $\bar{T}=600$ DP runtime.
At $\bar{T}=120$, all plotted settings retain at least $99\%$ of the optimal performance while using at most $4\%$ of the $\bar{T}=600$ runtime.
Therefore, the stationary-suffix structure makes the DP usable in regimes where solving a long horizon exactly would be unnecessarily expensive.
In Appendix~\ref{sec:beam}, we show that a beam search-based algorithm can speed up the truncated DP solution while preserving much of the optimality, demonstrating how such algorithms can be used in practice. 

%% file: sections_arxiv/related.tex
\section{Related Works}
\paragraph{Human-AI Collaboration} Work in human-AI collaboration develops models outlining human and AI interactions~\citep{no_free_lunch_human_ai,algorithmic_delegates}, grounded in real-world studies of these phenomena~\citep{intuition_human_ai_decisions,principles_limits_algorithm}.
On the modeling side, prior work establishes how decision quality can be decomposed into various factors, including communication between humans and AI~\citep{explanations_means_to_end}, the level of complementarity between humans and AI~\citep{human_expertise}, and human reliance on AI~\citep{reliance_decisions}.
Empirical studies on explanations and over-reliance motivate theoretical models of how humans update beliefs under AI guidance~\citep{cognitive_forcing_functions,application_grounded_explanation}.
Such systems produce not only short-term decisions, but also longer-term learning effects which impact what a human knows~\citep{humans_learn}. 
Within this line of work, most directly related is~\citet{ai_assisted_decisions}, which demonstrates the optimality of stationary policies for human learning under the assumption that features are uncorrelated; we show that once feature correlation is introduced, stationary policies can perform arbitrarily poorly, and derive optimal algorithms in this scenario.

\paragraph{Explore-then-Commit} A popular framework for balancing exploration and exploitation is ``explore-then-commit'', widely used in the bandit literature~\citep{explore_then_commit,bandit_books}. 
The key idea is to explore for a fixed budget of time, then exploit the best discovered strategy, resulting in sublinear $O(T^{2/3})$ regret. 
Our setting differs from traditional multi-armed bandit setting because exploration teaches humans about different features rather than learning arm rewards. 
This changes the problem significantly: the effective reward of each arm evolves endogenously with the recommendation policy, as the human's imputation of unseen features updates through experience.
In this sense, our work is related to correlated arms ~\citep{correlated_arms}, with the central distinction that human interactions with tests evolve as they learn, which contrasts with correlated but fixed arms studied in the literature. 

\paragraph{Machine Teaching} 
The machine teaching literature focuses on how to best select training to distill information to a student model~\citep{machine_teaching}.
Work in this field establishes how many examples are needed to teach a function by establishing a ``teaching dimension''~\citep{complexity_teaching,teaching_computational_learning} which has influenced tutoring systems in education~\citep{machine_teaching_education}. 
Our work differs in two ways: we focus on settings where the AI teacher selects tests rather than training samples and features are correlated so revealing one gives information about another. 

%% file: sections_arxiv/discussion.tex
\section{Discussion}
\label{sec:discussion}
\paragraph{Relaxing Assumptions}
While our results rely on the assumption that parameters are fully known, we discuss how our results hold when this assumption is relaxed. 
In Appendix~\ref{sec:misspecification}, we demonstrate that our results are robust to misspecification in quantities such as $a$ and $\mathcal L$. 
Additionally, while we assume access to a parametric learning rule $\phi$, we do not impose a specific function. 
Instead, in practice, $\phi$ can be retrospectively fit to human learning data, making the model broadly applicable.

\paragraph{Limitations}
First, our model captures test cost through a uniform cardinality budget \(|S_t|\le k\), but does not distinguish heterogeneous costs, delays, or risks.
Incorporating such costs would introduce a tradeoff between information value and resource use, potentially changing the optimal policy.
Second, the AI still needs enough information to evaluate how a test set affects the human's future predictions.
Our misspecification experiments suggest robustness to moderate errors, while formal regret or robustness guarantees are left to future work.
Finally, the framework can extend to structured nonlinear models, such as \(y=f(\sum_i a_i x_i)+\varepsilon\), but general nonlinear relationships may require richer information about the conditional distribution of unseen features.

\paragraph{Explore-then-commit in other Human-AI Settings} 
Our results demonstrate the importance of exploration in human-AI collaborations, which is an important but missing part of many human-AI collaboration paradigms. 
For example, the  indistinguishability framework~\citep{human_expertise} relies on human knowledge to assist AI, and joint exploration between humans and AI could reveal further information on where exactly humans have expertise. 
More broadly, any human-AI system where AI shapes human exposure should account for this exploration-commitment tradeoff.


%% file: sections_arxiv/appendix.tex
\section{Speeding up Dynamic Programming with Beam Search}
\label{sec:beam}
The exact forward DP is layered by prefix length: layer \(t\) contains the reachable count states \(\mathbf m\) with \(\sum_i m_i=t\), together with the best accumulated discounted prefix loss \(V_t(\mathbf m)\) and a backpointer for one prefix attaining it.
Storing every state in each layer can become expensive as $n$, $k$, or $T$ grows.
We therefore evaluate a beam-search variant of the truncated DP.
After forming the successor layer \(t+1\), the algorithm ranks the candidate prefix states and keeps only the best \(M\) states before the next expansion.
Here a prefix state is a count vector \(\mathbf m\) at the end of a length-\(t+1\) action prefix, together with its best prefix loss and backpointer; \(M\) is the beam width.
The pruning score for a candidate state \(\mathbf m\) is its accumulated prefix loss plus a closed-form estimate of the remaining discounted loss through \(T=600\) obtained by appending the best stationary suffix from \(\mathbf m\).
Thus states are retained when they are promising for the full evaluation horizon rather than only for short-run loss.
This beam search is not guaranteed to preserve the globally optimal path, but it greatly reduces the number of prefix states expanded.

We evaluate this approximation with a matched-runtime comparison on larger feature-selection problems.
The left panel of Figure~\ref{fig:beam-vs-truncated} uses \((n,k)=(10,2),(15,3),(20,4),(25,5)\), so the plotted settings keep \(n/k=5\) while increasing the full action spaces up to \(\binom{25}{5}=53{,}130\).
Exact finite-horizon DP at \(T=600\) is infeasible for these instances, so we compare beam search directly against truncated DP under matched runtimes.
For each beam runtime, we take the best truncated-DP policy whose runtime is no larger on the same random instance, and report the ratio of this time-matched truncated loss to the beam loss.
Values above \(100\%\) therefore mean that beam search achieves lower loss than truncated DP within the same time budget.
All runs use the full \(\binom{n}{k}\) action space, \(\rho=0.99\), \(\alpha=1.05\), and three random draws of \(a,h_0\sim\mathrm{Uniform}[0,1]\).
Thus Figure~\ref{fig:beam-vs-truncated} reports beam-search performance on larger full-action instances.

\section{Misspecification Experiments}
\label{sec:misspecification}
We run experiments to understand how robust our results are to misspecification in parameters. 
We use a fixed two-feature instance rather than resampling instances: $n=2$, $k=1$, $a=(1.0,0.8)$, $\hat a^{(0)}=(0.2,1.3)$, $\rho=0.8$, $\delta=0.99$, $\varepsilon_{\mathrm{noise}}=10^{-3}$, and $\phi(m)=1.10^{-2m}$.
For each trial, the AI first computes a DP policy using possibly misspecified inputs, and we then evaluate the resulting policy in the unperturbed true model.
We perturb one input at a time: $\hat a^{(0)}$ by a random vector with norm $\eta\|\hat a^{(0)}\|_2$; $\phi$ by replacing $\log \alpha$ with $(1\pm\eta)\log\alpha$; the loss oracle $\mathcal L$ by replacing the equicorrelation parameter $\rho$ with $\rho(1\pm\eta)$ only in the covariance loss used to solve the DP; and the conditional-imputation rule $\mathbb E[\mathbf x\mid \mathbf x_S]$ by replacing the conditional-mean coefficient $\rho$ with $\rho(1\pm\eta)$ only in the imputation map used to solve the DP. Perturbed correlations are clipped to keep $|\rho|<1$.
For each perturbation level $\eta$, the AI computes the $T=600$ DP policy using these perturbed inputs and then deploys the recovered stationary-suffix policy in the true model for the same horizon.
Thus the imputation perturbation represents imperfect AI knowledge of the conditional-imputation rule, not humans imputing incorrectly during prediction; the $\mathcal L$ perturbation isolates the effect of misspecifying the covariance parameter in the loss oracle.
For every perturbation type and every value of $\eta$, we run $80$ repeats. In each repeat, scalar perturbations sample the sign uniformly from $\{+1,-1\}$, while $\hat a^{(0)}$ perturbations sample an independent random direction.

\begin{figure}
    \centering
    \begin{minipage}{0.49\linewidth}
        \centering
        \includegraphics[width=\linewidth]{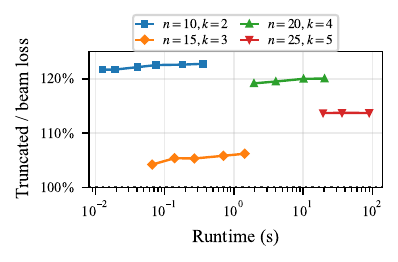}
    \end{minipage}
    \hfill
    \begin{minipage}{0.49\linewidth}
        \centering
        \includegraphics[width=\linewidth]{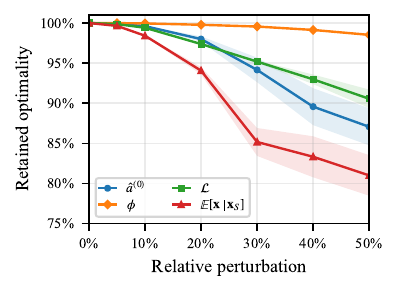}
    \end{minipage}
    \caption{Beam-search scaling and parameter-misspecification robustness. Left: matched-runtime comparison between beam search and truncated DP at evaluation horizon $T=600$ for full-action instances \((n,k)=(10,2),(15,3),(20,4),(25,5)\), with \(\rho=0.99\), \(\alpha=1.05\), and three random draws of \((a,\hat a^{(0)})\). The plotted value is the loss ratio of the best time-matched truncated-DP policy to the beam policy, so values above \(100\%\) indicate lower loss for beam search. Right: retained optimality after solving the \(T=600\) DP with one misspecified planning input at perturbation level \(\eta\), then evaluating the recovered stationary-suffix policy in the true fixed \((n,k)=(2,1)\) model. Right-panel bands show mean \(\pm 1.96\) standard errors over 80 perturbation repeats.}
    \label{fig:beam-vs-truncated}
    \label{fig:parameter-misspecification-robustness}
\end{figure}

The right panel of Figure~\ref{fig:parameter-misspecification-robustness} shows that the dynamic policy is robust to moderate misspecification.
We are most robust to perturbations in $\phi$; even at 50\% perturbation, retained optimality is only impacted by 1.5\%.
Conversely, we are least robust to the conditional-imputation rule; at $\eta=20\%$, this leads to a 5.9\% drop, and at $\eta=50\%$, this leads to a 19.0\% drop in retained optimality.
The loss oracle is less sensitive than the conditional-imputation rule but still matters at large perturbations, with a 9.4\% drop at $\eta=50\%$.
At moderate levels of perturbation, such as $\eta=10\%$, all four perturbation types retain at least $98.4\%$ of the reference dynamic policy's performance on average.
This suggests that the policy's value is not driven by exact recovery of every early action, but by avoiding severe misspecification of the correlation and conditional-imputation structure.

\input{sections_arxiv/arxiv_experiments}

\section{Proofs}
\label{sec:proofs}

\subsection{Gaussian One-Step Loss Identity}
\label{app:one-step-loss}

\begin{proposition}[One-step loss under correlated features]
\label{prop:one-step-loss-corr}
Assume \(x^{(t)}\sim\mathcal N(0,\Sigma)\) with \(\Sigma\succ0\), and assume the noise \(\varepsilon^{(t)}\) is independent of \(x^{(t)}\), has mean zero, and has variance \(\sigma_\varepsilon^2\).
Fix a round $t$ and a selected subset $S_t \subseteq [n]$, and let
$\bar{S}_t = [n] \setminus S_t$. Define the belief error vectors
\[
    e^{(t)} := a - \hat{a}^{(t)}, \qquad
    e^{(t)}_{S_t} := a_{S_t} - \hat{a}^{(t)}_{S_t}, \qquad
    e^{(t)}_{\bar{S}_t} := a_{\bar{S}_t} - \hat{a}^{(t)}_{\bar{S}_t},
\]
and the conditional covariance
\[
    \Sigma_{\bar{S}_t \mid S_t}
    := \Sigma_{\bar{S}_t\bar{S}_t}
    - \Sigma_{\bar{S}_t S_t}\Sigma_{S_t S_t}^{-1}\Sigma_{S_t\bar{S}_t}.
\]
Then the one-step expected loss
\[
    \mathcal{L}(S_t, \hat{a}^{(t)})
    := \mathbb{E}\!\left[\left(y^{(t)} - \hat{y}^{(t)}\right)^2\right]
\]
admits the closed form
\[
\boxed{
    \mathcal{L}(S_t, \hat{a}^{(t)})
    =
    \sigma_\varepsilon^2
    +
    a_{\bar{S}_t}^\top \Sigma_{\bar{S}_t \mid S_t}\, a_{\bar{S}_t}
    +
    \Bigl(e^{(t)}_{S_t}
    + \Sigma_{S_t S_t}^{-1}\Sigma_{S_t\bar{S}_t}\,e^{(t)}_{\bar{S}_t}\Bigr)^\top
    \Sigma_{S_t S_t}
    \Bigl(e^{(t)}_{S_t}
    + \Sigma_{S_t S_t}^{-1}\Sigma_{S_t\bar{S}_t}\,e^{(t)}_{\bar{S}_t}\Bigr).
}
\]
\end{proposition}
\begin{proof}[Proof of Proposition~\ref{prop:one-step-loss-corr}]
For notational simplicity, write
$$
    S := S_t,\qquad \bar S := \bar S_t,\qquad
    \hat a := \hat a^{(t)},\qquad e := e^{(t)},\qquad
    x := x^{(t)},\qquad \varepsilon:=\varepsilon^{(t)}.
$$
By the Gaussian conditional-imputation rule for the unselected coordinates, we use
$$
    \mathbb E[x_{\bar S}\mid x_S] = \Sigma_{\bar S S}\Sigma_{SS}^{-1}x_S.
$$
Therefore the prediction after selecting the coordinates in \(S\) is
$$
    \hat y^{(t)}
    =
    \hat a_S^\top x_S
    +
    \hat a_{\bar S}^\top
    \Sigma_{\bar S S}\Sigma_{SS}^{-1}x_S.
$$
On the other hand, the ground truth is
$$
    y^{(t)} = a_S^\top x_S+a_{\bar S}^\top x_{\bar S}+\varepsilon.
$$
Hence the prediction error is
\[
\begin{aligned}
    y^{(t)}-\hat y^{(t)}
    &=
    a_S^\top x_S+a_{\bar S}^\top x_{\bar S}+\varepsilon
    -
    \hat a_S^\top x_S
    -
    \hat a_{\bar S}^\top
    \Sigma_{\bar S S}\Sigma_{SS}^{-1}x_S  \\
    &=
    \varepsilon
    +
    e_S^\top x_S
    +
    a_{\bar S}^\top x_{\bar S}
    -
    \hat a_{\bar S}^\top
    \Sigma_{\bar S S}\Sigma_{SS}^{-1}x_S .
\end{aligned}
\]
Since \(e_{\bar S}=a_{\bar S}-\hat a_{\bar S}\), we may rewrite $-\hat a_{\bar S} = e_{\bar S}-a_{\bar S}.$ Thus
\[
\begin{aligned}
    y^{(t)}-\hat y^{(t)}
    &=
    \varepsilon
    +
    e_S^\top x_S
    +
    a_{\bar S}^\top x_{\bar S}
    +
    (e_{\bar S}-a_{\bar S})^\top
    \Sigma_{\bar S S}\Sigma_{SS}^{-1}x_S \\
    &=
    \varepsilon
    +
    a_{\bar S}^\top
    \Bigl(
        x_{\bar S}
        -
        \Sigma_{\bar S S}\Sigma_{SS}^{-1}x_S
    \Bigr)
    +
    e_S^\top x_S
    +
    e_{\bar S}^\top
    \Sigma_{\bar S S}\Sigma_{SS}^{-1}x_S .
\end{aligned}
\]
Using symmetry of \(\Sigma\), we have
\[
    e_{\bar S}^\top
    \Sigma_{\bar S S}\Sigma_{SS}^{-1}x_S
    =
    \Bigl(
        \Sigma_{SS}^{-1}\Sigma_{S\bar S}e_{\bar S}
    \Bigr)^\top x_S.
\]
Therefore
\[
    y^{(t)}-\hat y^{(t)}
    =
    \varepsilon
    +
    a_{\bar S}^\top r_{\bar S\mid S}
    +
    \Bigl(
        e_S+
        \Sigma_{SS}^{-1}\Sigma_{S\bar S}e_{\bar S}
    \Bigr)^\top x_S,
\]
where $ r_{\bar S\mid S} := x_{\bar S} - \Sigma_{\bar S S}\Sigma_{SS}^{-1}x_S.$

We now compute the covariance of \(r_{\bar S\mid S}\). First,
\[
\begin{aligned}
    \mathbb E[r_{\bar S\mid S}r_{\bar S\mid S}^\top]
    &=
    \mathbb E\Big[
        \bigl(x_{\bar S}
        -
        \Sigma_{\bar S S}\Sigma_{SS}^{-1}x_S\bigr)
        \bigl(x_{\bar S}
        -
        \Sigma_{\bar S S}\Sigma_{SS}^{-1}x_S\bigr)^\top
    \Big] \\
    &=
    \Sigma_{\bar S\bar S}
    -
    \Sigma_{\bar S S}\Sigma_{SS}^{-1}\Sigma_{S\bar S}
    -
    \Sigma_{\bar S S}\Sigma_{SS}^{-1}\Sigma_{S\bar S}
    +
    \Sigma_{\bar S S}\Sigma_{SS}^{-1}\Sigma_{SS}
    \Sigma_{SS}^{-1}\Sigma_{S\bar S} \\
    &=
    \Sigma_{\bar S\bar S}
    -
    \Sigma_{\bar S S}\Sigma_{SS}^{-1}\Sigma_{S\bar S} \\
    &=
    \Sigma_{\bar S\mid S}.
\end{aligned}
\]
Also,
\[
\begin{aligned}
    \mathbb E[r_{\bar S\mid S}x_S^\top]
    &=
    \mathbb E\Big[
        \bigl(
            x_{\bar S}
            -
            \Sigma_{\bar S S}\Sigma_{SS}^{-1}x_S
        \bigr)x_S^\top
    \Big] \\
    &=
    \Sigma_{\bar S S}
    -
    \Sigma_{\bar S S}\Sigma_{SS}^{-1}\Sigma_{SS} \\
    &=
    0.
\end{aligned}
\]
Thus the residual \(r_{\bar S\mid S}\) is uncorrelated with the selected
coordinates \(x_S\).

Let
\[
    b_S
    :=
    e_S+
    \Sigma_{SS}^{-1}\Sigma_{S\bar S}e_{\bar S}.
\]
Then
\[
    y^{(t)}-\hat y^{(t)}
    =
    \varepsilon
    +
    a_{\bar S}^\top r_{\bar S\mid S}
    +
    b_S^\top x_S.
\]
The noise is independent of \(x\) and has mean zero, so all cross terms
involving \(\varepsilon\) have expectation zero, while
\(\mathbb E[\varepsilon^2]=\sigma_\varepsilon^2\).
Therefore
\[
\begin{aligned}
    \mathcal L(S_t,\hat a^{(t)})
    &=
    \mathbb E\!\left[
        \left(
            a_{\bar S}^\top r_{\bar S\mid S}
            +
            \varepsilon
            +
            b_S^\top x_S
        \right)^2
    \right] \\
    &=
    \sigma_\varepsilon^2
    +
    a_{\bar S}^\top
    \mathbb E[r_{\bar S\mid S}r_{\bar S\mid S}^\top]
    a_{\bar S}
    +
    b_S^\top
    \mathbb E[x_Sx_S^\top]
    b_S
    +
    2a_{\bar S}^\top
    \mathbb E[r_{\bar S\mid S}x_S^\top]
    b_S \\
    &=
    \sigma_\varepsilon^2
    +
    a_{\bar S}^\top
    \Sigma_{\bar S\mid S}
    a_{\bar S}
    +
    b_S^\top
    \Sigma_{SS}
    b_S.
\end{aligned}
\]
Substituting the definition of \(b_S\) gives
\[
    \mathcal L(S_t,\hat a^{(t)})
    =
    \sigma_\varepsilon^2
    +
    a_{\bar S_t}^\top \Sigma_{\bar S_t\mid S_t}a_{\bar S_t}
    +
    \Bigl(
        e^{(t)}_{S_t}
        +
        \Sigma_{S_tS_t}^{-1}\Sigma_{S_t\bar S_t}
        e^{(t)}_{\bar S_t}
    \Bigr)^\top
    \Sigma_{S_tS_t}
    \Bigl(
        e^{(t)}_{S_t}
        +
        \Sigma_{S_tS_t}^{-1}\Sigma_{S_t\bar S_t}
        e^{(t)}_{\bar S_t}
    \Bigr),
\]
which is the desired closed form.
\end{proof}

\subsection{Symmetric Two-Feature Example}
\label{app:symmetric-two-feature-proof}

\begin{proof}[Proof of Proposition~\ref{thm:symmetric-two-feature-geometric}]
For a state with current coefficient-error vector $(e_1^{(t)},e_2^{(t)})$, Proposition~\ref{prop:one-step-loss-corr} gives
\[
\mathcal L(\{1\},\hat{a}^{(t)})=\sigma_\varepsilon^2+(1-\rho^2)a^2+(e_1^{(t)}+\rho e_2^{(t)})^2,
\qquad
\mathcal L(\{2\},\hat{a}^{(t)})=\sigma_\varepsilon^2+(1-\rho^2)a^2+(e_2^{(t)}+\rho e_1^{(t)})^2.
\]
Without loss of generality, assume that $m_{1}(t) \geq m_{2}(t)$, and note that the two features are symmetric.
Let
\[
r:=\min\{m_1(t),m_2(t)\},
\qquad
d:=|m_1(t)-m_2(t)|,
\]
and let $V^*(r,d)$ be the optimal infinite-horizon discounted loss from this reduced state. 
In this symmetric instance, let \(E:=a-\hat a_i^{(0)}\), which is the same for both features and equals \(a\) under the proposition's assumption \(\hat a_i^{(0)}=0\).
Since $e_i^{(t)}=E\alpha^{-m_i(t)}$, if the exposure counts are $(m_1(t),m_2(t))=(r+d,r)$ with $d\ge 0$, then
\[
e_1^{(t)}=E\alpha^{-(r+d)},
\qquad
e_2^{(t)}=E\alpha^{-r}.
\]
Let $\lambda:=\alpha^{-2}$, $c:=\sigma_\varepsilon^2+(1-\rho^2)a^2$, and
\[
    \rho^*(\delta):=\sqrt{\frac{1-\delta}{1-\delta/\alpha^2}}.
\]
Hence the one-step losses are
\[
\mathcal L_\mathrm{keep}(r,d)=c+E^2\lambda^r(\alpha^{-d}+\rho)^2,
\qquad
\mathcal L_\mathrm{switch}(r,d)=c+E^2\lambda^r(1+\rho\alpha^{-d})^2,
\]
where ``keep'' means selecting the more-exposed feature (feature 1) and ``switch'' means selecting the less-exposed feature (feature 2).

We next write 
\[
V^*(r,d)=\frac{c}{1-\delta}+E^2\lambda^rF_d^*,
\]
where for every $d\ge 1$,
\begin{equation}
\label{eq:bellman-reduced}
F_d^*=
\min\Big\{
(\alpha^{-d}+\rho)^2+\delta F_{d+1}^*,
\ (1+\rho\alpha^{-d})^2+\delta\lambda F_{d-1}^*
\Big\}.
\end{equation}
The initial condition is
\[
F_0^*=(1+\rho)^2+\delta F_1^*.
\]
Define the Bellman operator $T$ on bounded sequences $F=(F_d)_{d\ge 0}$ and $G=(G_d)_{d\ge 0}$ by the right-hand side of
\eqref{eq:bellman-reduced}. Then
\[
\|TF-TG\|_\infty\le \delta\|F-G\|_\infty
\]
for all bounded $F,G$, because the continuation terms are multiplied by either $\delta$ or $\delta\lambda\le \delta$.
Hence $T$ is a $\delta$-contraction on $\ell_\infty$, so it has a unique fixed point, namely $F^*$.

\medskip
\noindent
\textbf{Case 1: $\rho\le \rho^*(\delta)$.}
Consider the policy that always chooses the more-exposed feature once a gap has opened. Its continuation value from state $d$ is
\[
F_d^{\mathrm{same}}
:=\sum_{s=0}^\infty \delta^s(\rho+\alpha^{-(d+s)})^2.
\]
Expanding the square and summing the resulting geometric series gives
\begin{equation}
\label{eq:Fsame-closed}
F_d^{\mathrm{same}}
=
\frac{\rho^2}{1-\delta}
+\frac{2\rho\alpha^{-d}}{1-\delta/\alpha}
+\frac{\alpha^{-2d}}{1-\delta/\alpha^2}.
\end{equation}
In particular,
\[
F_d^{\mathrm{same}}=(\alpha^{-d}+\rho)^2+\delta F_{d+1}^{\mathrm{same}}
\qquad\text{for every }d\ge 0.
\]
So $F^{\mathrm{same}}$ satisfies the first Bellman branch with equality.

We now compare the two Bellman branches. For $d\ge 1$,
\begin{align}
&\Big[(1+\rho\alpha^{-d})^2+\delta\lambda F_{d-1}^{\mathrm{same}}\Big]
-\Big[(\alpha^{-d}+\rho)^2+\delta F_{d+1}^{\mathrm{same}}\Big] \notag\\
&\qquad=
\frac{\big[(1-\delta/\alpha^2)-\alpha^{-2d}(1-\delta)\big]\big[(1-\rho^2)-\delta(1-\rho^2/\alpha^2)\big]}{(1-\delta)(1-\delta/\alpha^2)}.
\label{eq:same-branch-diff}
\end{align}
Because $0<\alpha^{-1}<1$ and $d\ge 1$,
\[
(1-\delta/\alpha^2)-\alpha^{-2d}(1-\delta)
\ge
(1-\delta/\alpha^2)-\alpha^{-2}(1-\delta)
=1-\alpha^{-2}>0.
\]
Also, $(1-\delta)(1-\delta/\alpha^2)>0$. Therefore the sign of \eqref{eq:same-branch-diff} is determined by
\[
(1-\rho^2)-\delta(1-\rho^2/\alpha^2).
\]
If $\rho\le \rho^*(\delta)$, then this term is nonnegative, so the first Bellman branch is weakly smaller than the second for every $d\ge 1$.
Hence $F^{\mathrm{same}}$ attains the Bellman minimum at every state and is therefore a fixed point of $T$.
By uniqueness of the fixed point, $F^{\mathrm{same}}=F^*$.

Starting from the symmetric state $d=0$, the first move is a tie by symmetry. Once one feature is chosen, the state becomes $d=1$, and the Bellman-optimal action is to keep choosing the more-exposed feature forever. Thus an optimal policy is \((\{1\},\{1\},\{1\},\dots)\) or symmetrically \((\{2\},\{2\},\{2\},\dots)\).

\medskip
\noindent
\textbf{Case 2: $\rho\ge \rho^*(\delta)$.}
Now consider the policy that always chooses the less-exposed feature whenever $d\ge 1$.
Starting from $d=0$, one feature must first be chosen arbitrarily, which moves the system to $d=1$.
Therefore
\[
F_0^{\mathrm{alt}}=(1+\rho)^2+\delta\Big((1+\rho/\alpha)^2+\delta\lambda F_0^{\mathrm{alt}}\Big),
\]
and solving for $F_0^{\mathrm{alt}}$ yields
\begin{equation}
\label{eq:F0alt}
F_0^{\mathrm{alt}}=
\frac{(1+\rho)^2+\delta(1+\rho/\alpha)^2}{1-\delta^2\lambda}.
\end{equation}
For $d\ge 1$, if the policy always chooses the less-exposed feature, then after $d$ such choices the gap closes and the process returns to state $0$. Hence
\begin{equation}
\label{eq:Falt-sum}
F_d^{\mathrm{alt}}
=
(\delta\lambda)^dF_0^{\mathrm{alt}} + \sum_{j=0}^{d-1}(\delta\lambda)^j(1+\rho\alpha^{-(d-j)})^2.
\end{equation}
In particular,
\[
F_d^{\mathrm{alt}}=(1+\rho\alpha^{-d})^2+\delta\lambda F_{d-1}^{\mathrm{alt}}
\qquad\text{for every }d\ge 1,
\]
and also
\[
F_0^{\mathrm{alt}}=(1+\rho)^2+\delta F_1^{\mathrm{alt}}.
\]
So $F^{\mathrm{alt}}$ satisfies the second Bellman branch with equality.

Define
\[
Q_d:=\sum_{j=0}^{d-1}\alpha^{-2j}\sum_{\ell=0}^{j}\delta^\ell>0.
\]
Substituting \eqref{eq:F0alt}--\eqref{eq:Falt-sum} into the Bellman comparison gives, for every $d\ge 1$,
\begin{align}
&\Big[(\alpha^{-d}+\rho)^2+\delta F_{d+1}^{\mathrm{alt}}\Big]
-\Big[(1+\rho\alpha^{-d})^2+\delta\lambda F_{d-1}^{\mathrm{alt}}\Big] \notag\\
&\qquad=
(1-\alpha^{-2})Q_d\Big[\delta(1-\rho^2/\alpha^2)-(1-\rho^2)\Big].
\label{eq:alt-branch-diff}
\end{align}
Since $1-\alpha^{-2}>0$ and $Q_d>0$, the sign of \eqref{eq:alt-branch-diff} is determined by
\[
\delta(1-\rho^2/\alpha^2)-(1-\rho^2).
\]
If $\rho\ge \rho^*(\delta)$, then this term is nonnegative, so the second Bellman branch is weakly smaller than the first for every $d\ge 1$.
Hence $F^{\mathrm{alt}}$ attains the Bellman minimum at every state and is therefore a fixed point of $T$.
By uniqueness of the fixed point, $F^{\mathrm{alt}}=F^*$.

Starting from the symmetric state $d=0$, the first move is again a tie. Thereafter the Bellman-optimal action is always to choose the less-exposed feature, which forces the process to alternate forever. Thus an optimal policy is \((\{1\},\{2\},\{1\},\{2\},\dots)\) or symmetrically \((\{2\},\{1\},\{2\},\{1\},\dots)\).

\medskip
\noindent
\textbf{Case 3: $\rho=\rho^*(\delta)$.}
At $\rho=\rho^*(\delta)$, the comparison terms in both \eqref{eq:same-branch-diff} and \eqref{eq:alt-branch-diff} are zero, so both candidate policies attain the Bellman minimum at every state. Therefore both are optimal.
The value function remains unique because the Bellman operator is still a contraction, but the optimal policy need not be unique.
\end{proof}

\subsection{Negative Results for Stationary Policy}
\label{app:negative-results}

\begin{proof}[Proof of Theorem~\ref{thm:stationary_bad}]
Since $\Sigma$ has a nonzero off-diagonal entry, there exist $i\neq j$ such
that $\Sigma_{ij}\neq 0$. Choose any set
\[
T\subseteq[n]\setminus\{j\},\qquad |T|=k,\qquad i\in T .
\]
This is possible because $k<n$.

For disjoint sets $A,S\subseteq[n]$, write
\[
\Sigma_{A\mid S}
:=
\Sigma_{AA}-\Sigma_{AS}\Sigma_{SS}^{-1}\Sigma_{SA},
\]
with the convention $\Sigma_{A\mid \emptyset}:=\Sigma_{AA}$. Let
$\mathbf 1_A$ denote the all-ones vector indexed by $A$.
Set the signal scale
\[
    \beta_\delta^2
    :=
    \begin{cases}
    1, & \sigma_\varepsilon^2=0,\\[2pt]
    \sigma_\varepsilon^2/(1-\delta), & \sigma_\varepsilon^2>0.
    \end{cases}
\]

We construct the instance as follows. The true coefficients are
\[
a_\ell =
\begin{cases}
\beta_\delta, & \ell\in T,\\
0, & \ell\notin T.
\end{cases}
\]
The initial belief is correct on every coordinate except $j$:
\[
\hat a_\ell^{(0)}=a_\ell\quad\text{for }\ell\neq j,
\qquad
\hat a_j^{(0)}=-\beta_\delta.
\]
Thus the initial coefficient error $e^{(0)}:=a-\hat a^{(0)}$ satisfies
$e_j^{(0)}=\beta_\delta$ and $e_\ell^{(0)}=0$ for every $\ell\neq j$.

Use perfect-after-one-exposure learning:
\[
\phi(0)=1,\qquad \phi(m)=0\quad\text{for all }m\ge 1.
\]
Equivalently, once a feature is shown once, its coefficient error becomes zero
forever.

We first upper bound the optimal dynamic loss. Consider the nonstationary
policy that shows feature $j$ in the first round and then shows $T$ forever:
\[
S_0=\{j\},\qquad S_t=T\quad\text{for all }t\ge 1 .
\]
After the first round, feature $j$'s coefficient error is learned away.
Moreover, the coefficients on $T$ were already correct initially, and all
coefficients outside $T$ are zero. Therefore, from time $t\ge 1$, showing
$T$ incurs only the irreducible noise loss $\sigma_\varepsilon^2$.

By Proposition~\ref{prop:one-step-loss-corr}, the policy-dependent part of
the first-round loss is
\[
B_{\Sigma,k}
:=
\mathbf 1_T^\top \Sigma_{T\mid \{j\}}\mathbf 1_T+\Sigma_{jj}.
\]
Thus
\[
\inf_{\pi\in\Pi} J(\delta,\pi)
\le
\frac{\sigma_\varepsilon^2}{1-\delta}
+
\beta_\delta^2 B_{\Sigma,k}.
\]

We now lower bound every stationary policy. Consider any stationary policy
that always shows some set $S\subseteq[n]$, $|S|\le k$.

First suppose $S=T$. Then feature $j$ is never shown, so its coefficient
error remains equal to \(\beta_\delta\) forever. Since $i\in T$ and $\Sigma_{ij}\neq 0$,
we have $\Sigma_{Tj}\neq 0$. Define
\[
q_{\Sigma,k}
:=
\Sigma_{jT}\Sigma_{TT}^{-1}\Sigma_{Tj}.
\]
Because $\Sigma_{TT}^{-1}$ is positive definite and $\Sigma_{Tj}\neq 0$, we
have $q_{\Sigma,k}>0$. Again by
Proposition~\ref{prop:one-step-loss-corr}, every round under the stationary
policy $T^\infty$ has policy-dependent loss at least
\(\beta_\delta^2 q_{\Sigma,k}\), in addition to the noise
\(\sigma_\varepsilon^2\).

Now suppose $S\neq T$. Since $|S|\le k=|T|$, the set $S$ cannot contain all
of $T$. Hence $T\setminus S\neq\emptyset$. Since the true coefficient vector
is nonzero on $T\setminus S$, the irreducible conditional-variance term in
the policy-dependent part of the one-step loss is at least
\[
\beta_\delta^2
\mathbf 1_{T\setminus S}^\top
\Sigma_{T\setminus S\mid S}
\mathbf 1_{T\setminus S}.
\]
Because $\Sigma\succ 0$, every Schur complement
$\Sigma_{T\setminus S\mid S}$ is positive definite. Therefore this quantity
is strictly positive.

There are only finitely many feasible stationary sets $S$. Define
\[
r_{\Sigma,k}
:=
\min_{\substack{S\subseteq[n],\, |S|\le k\\ S\neq T}}
\mathbf 1_{T\setminus S}^\top
\Sigma_{T\setminus S\mid S}
\mathbf 1_{T\setminus S}.
\]
The preceding argument shows that $r_{\Sigma,k}>0$.

Therefore every stationary policy has policy-dependent per-round loss at least
\[
d_{\Sigma,k}:=\min\{q_{\Sigma,k},r_{\Sigma,k}\}>0
\]
from time $t=1$ onward. Hence every stationary policy satisfies
\[
J(\delta,\pi)
\ge
\frac{\sigma_\varepsilon^2}{1-\delta}
+
\sum_{t=1}^\infty \delta^t \beta_\delta^2 d_{\Sigma,k}
=
\frac{\sigma_\varepsilon^2}{1-\delta}
+
\frac{\delta}{1-\delta}\beta_\delta^2 d_{\Sigma,k}.
\]
Combining the upper and lower bounds, if \(\sigma_\varepsilon^2=0\), then
\[
\frac{\inf_{\pi\in\Pi_{\rm stat}}J(\delta,\pi)}
     {\inf_{\pi\in\Pi}J(\delta,\pi)}
\ge
\frac{\delta}{1-\delta}\cdot
\frac{d_{\Sigma,k}}{B_{\Sigma,k}}.
\]
If \(\sigma_\varepsilon^2>0\), then our choice
\(\beta_\delta^2=\sigma_\varepsilon^2/(1-\delta)\) gives
\[
\frac{\inf_{\pi\in\Pi_{\rm stat}}J(\delta,\pi)}
     {\inf_{\pi\in\Pi}J(\delta,\pi)}
\ge
\frac{1+d_{\Sigma,k}\delta/(1-\delta)}{1+B_{\Sigma,k}}
\ge
\frac{\delta}{1-\delta}\cdot
\frac{d_{\Sigma,k}}{1+B_{\Sigma,k}}.
\]
Thus the theorem holds with
\[
c_{\Sigma,k}:=\frac{d_{\Sigma,k}}{1+B_{\Sigma,k}}>0.
\]
\end{proof}

\subsection{Eventually Stationary Suffix}
\label{app:eventual-static-suffix}

\begin{proof}[Proof of Theorem~\ref{thm:eventually-static-suffix-regular}]
For compactness, write
\[
    g(\mathbf m,S):=\mathcal L(S,\hat a(\mathbf m)),
    \qquad
    \theta^{(t)}:=\hat a(\mathbf m^{(t)}),
\]
Also write \(N:=\bar{R}=[n]\setminus R\) and \(\mathbf c=(c_i)_{i\in N}\), and let \(L_\ell\) be the Lipschitz constant from
Assumption~\ref{ass:regular-loss}. Since every feature in \(N\) is selected
only finitely many times, there exists \(T_R<\infty\) such that
\(S_s\subseteq R\) for all \(s\ge T_R\).

By coordinate-wise convergence, \(m_i^{(t)}\to\infty\) implies
\(\hat a_i(m_i^{(t)})\to a_i\) for every \(i\in R\). By definition of \(N\),
the coordinates in \(N\) freeze to their final counts \(c_i\) after time
\(T_R\). Hence
\[
    \theta^{(t)}\to \bar a(\mathbf a,\mathbf c).
\]
Define \(\bar{\mathcal L}_{\pi^\star}(S):=\mathcal L(S,\bar a(\mathbf a,\mathbf c))\).
Therefore, by the continuity part of
Assumption~\ref{ass:regular-loss}, for every feasible \(S\subseteq R\),
\[
    g(\mathbf m^{(t)},S)
    =
    \mathcal L(S,\theta^{(t)})
    \to
    \mathcal L(S,\bar a(\mathbf a,\mathbf c))
    =
    \bar{\mathcal L}_{\pi^\star}(S).
\]
The feasible action set \(\{S\subseteq R: |S|\le k\}\) is finite, so this
convergence is uniform over tail actions. Define
\[
    \varepsilon_t
    :=
    \max_{\substack{S\subseteq R\\ |S|\le k}}
    \left|g(\mathbf m^{(t)},S)-\bar{\mathcal L}_{\pi^\star}(S)\right|.
\]
Then \(\varepsilon_t\to0\).

We also define the maximum one-step learning increment remaining after time
\(t\):
\[
    \eta_t
    :=
    \max_{i\in R}
    \sup_{q\ge m_i^{(t)}}
    \left|\hat a_i(q+1)-\hat a_i(q)\right|,
\]
with the convention \(\eta_t=0\) if \(R=\emptyset\). Since each coordinate
learning sequence is convergent, its adjacent increments vanish, and hence
\(\eta_t\to0\).

Let
\[
    \mathcal M^\star
    :=
    \arg\min_{\substack{S\subseteq R\\ |S|\le k}}
    \bar{\mathcal L}_{\pi^\star}(S)
\]
be the set of limiting tail minimizers. If every feasible subset of \(R\)
belongs to \(\mathcal M^\star\), then the first claim is immediate after time
\(T_R\). Thus assume that at least one feasible tail action is not a minimizer,
and define
\[
    \gamma
    :=
    \min_{\substack{S\subseteq R,\ |S|\le k\\ S\notin\mathcal M^\star}}
    \left(
        \bar{\mathcal L}_{\pi^\star}(S)
        -
        \min_{\substack{T\subseteq R\\ |T|\le k}}
        \bar{\mathcal L}_{\pi^\star}(T)
    \right)
    >0,
\]
which is positive because the feasible tail-action set is finite.

Fix \(t\ge T_R\) and suppose, toward contradiction, that
\(S_t\notin\mathcal M^\star\). Choose any \(S^\star\in\mathcal M^\star\).
Construct an alternative policy \(\tilde\pi\) that agrees with \(\pi^\star\)
at every time except \(t\), selects \(S^\star\) at time \(t\) instead of
\(S_t\), and then follows the same future actions \(S_{t+1},S_{t+2},\ldots\)
as \(\pi^\star\). Let \(\tilde{\mathbf m}^{(s)}\) be the corresponding count
trajectory.

\paragraph{Current-round gain.}
By the definition of \(\varepsilon_t\),
\[
    g(\mathbf m^{(t)},S_t)-g(\mathbf m^{(t)},S^\star)
    \ge
    \bar{\mathcal L}_{\pi^\star}(S_t)
    -
    \bar{\mathcal L}_{\pi^\star}(S^\star)
    -2\varepsilon_t
    \ge
    \gamma-2\varepsilon_t.
\]

\paragraph{Future-round penalty.}
For every \(s>t\), the two count vectors \(\mathbf m^{(s)}\) and
\(\tilde{\mathbf m}^{(s)}\) differ only on coordinates in
\(S_t\triangle S^\star\). On each such coordinate, the two counts differ by
exactly one, and the smaller count is at least \(m_i^{(t)}\). Therefore
\[
    \|\hat a(\tilde{\mathbf m}^{(s)})-\hat a(\mathbf m^{(s)})\|_1
    \le
    |S_t\triangle S^\star|\eta_t
    \le
    2k\eta_t.
\]
Using Assumption~\ref{ass:regular-loss},
\[
\begin{aligned}
    \sum_{s=t+1}^{\infty}\delta^s
    \left|
        g(\tilde{\mathbf m}^{(s)},S_s)
        -
        g(\mathbf m^{(s)},S_s)
    \right|
    &\le
    \sum_{s=t+1}^{\infty}\delta^s L_\ell
    \|\hat a(\tilde{\mathbf m}^{(s)})-\hat a(\mathbf m^{(s)})\|_1\\
    &\le
    \delta^t\frac{2k\delta}{1-\delta}L_\ell\eta_t.
\end{aligned}
\]

\paragraph{Contradiction.}
Combining the current-round improvement with the future-round penalty gives
\[
    J(\delta,\tilde\pi)-J(\delta,\pi^\star)
    \le
    -\delta^t(\gamma-2\varepsilon_t)
    +
    \delta^t\frac{2k\delta}{1-\delta}L_\ell\eta_t.
\]
Because \(\varepsilon_t\to0\) and \(\eta_t\to0\), there exists \(T_d\ge T_R\)
such that
\[
    2\varepsilon_t+\frac{2k\delta}{1-\delta}L_\ell\eta_t<\gamma
\]
for all \(t\ge T_d\). For such \(t\), the alternative policy has strictly
smaller objective value than \(\pi^\star\), contradicting the optimality of
\(\pi^\star\). Hence \(S_t\in\mathcal M^\star\) for all \(t\ge T_d\).

If \(\bar{\mathcal L}_{\pi^\star}\) has a unique minimizer \(S^\star\), then
\(\mathcal M^\star=\{S^\star\}\), and the preceding conclusion gives
\(S_t=S^\star\) for all \(t\ge T_d\).
\end{proof}

\begin{corollary}[Two-feature asymmetry rules out perpetual alternation]
\label{cor:two-feature-general-phi-static}
In the Gaussian squared-loss model, fix \(n=2\), \(k=1\), and
\[
    \Sigma=
    \begin{pmatrix}
        1 & \rho\\
        \rho & 1
    \end{pmatrix},
    \qquad |\rho|<1.
\]
If the learning rule is convergent and \(|a_1|\ne |a_2|\), then every infinite-horizon optimal policy has a stationary suffix.
\end{corollary}

\begin{proof}[Proof of Corollary~\ref{cor:two-feature-general-phi-static}]
The Gaussian squared-loss model satisfies the hypotheses of
Theorem~\ref{thm:eventually-static-suffix-regular}: Proposition~\ref{prop:asymptotic-tail-loss}
gives the limiting tail losses, and Lemma~\ref{lem:uniform-lipschitz-loss}
verifies the Lipschitz regularity needed for the Gaussian squared-loss risk.
For \(n=2\), \(k=1\), a non-stationary policy would need both features to be
selected infinitely often. In that case \(R=\{1,2\}\) and \(N=\emptyset\). By
Proposition~\ref{prop:asymptotic-tail-loss},
\[
    \bar g_{\pi^\star}(\{1\})
    =
    \operatorname{Var}(y\mid x_1),
    \qquad
    \bar g_{\pi^\star}(\{2\})
    =
    \operatorname{Var}(y\mid x_2).
\]
Using
\[
    \Sigma=
    \begin{pmatrix}
        1 & \rho\\
        \rho & 1
    \end{pmatrix},
    \qquad
    y=a_1x_1+a_2x_2+\varepsilon,
\]
we obtain
\[
    \operatorname{Var}(y\mid x_1)=\sigma_\varepsilon^2+(1-\rho^2)a_2^2,
    \qquad
    \operatorname{Var}(y\mid x_2)=\sigma_\varepsilon^2+(1-\rho^2)a_1^2.
\]
Therefore
\[
    \bar g_{\pi^\star}(\{1\})-\bar g_{\pi^\star}(\{2\})
    =
    (1-\rho^2)(a_2^2-a_1^2).
\]
Since \(|\rho|<1\) and \(|a_1|\ne |a_2|\), the two tail losses are distinct.
Hence, whenever both features are recurrent, \(\bar g_{\pi^\star}\) has a
unique minimizer. Theorem~\ref{thm:eventually-static-suffix-regular} then rules out
perpetual alternation. If only one feature is recurrent, the policy is already
eventually stationary. Hence every optimal policy has a stationary suffix.
\end{proof}

\subsection{Tail Loss and Generic Uniqueness}
\label{app:tail-loss}

\begin{lemma}[Uniform Lipschitz continuity of the one-step loss]
\label{lem:uniform-lipschitz-loss}
In the Gaussian squared-loss model, suppose \(\mathcal X=\mathcal N(0,\Sigma)\) with \(\Sigma\succ0\).
Assume there is a finite \(B\) such that \(|a_i|\le B\) and \(|\hat{a}_i(r)|\le B\) for every $i\in[n]$ and
every $r\in\mathbb{N}_0$. 
There exists some $L_{\Sigma,B}$ such that for every pair of count states $\mathbf{m},\mathbf{m}' \in \mathbb{N}_0^{n}$ and every action $S$
with $|S|\le k$, we have that
\[
    |\mathcal L(S,\hat a(\mathbf{m})) - \mathcal L(S,\hat a(\mathbf{m}'))|
    \le
    L_{\Sigma,B}\,\|\hat{a}(\mathbf{m}) - \hat{a}(\mathbf{m}')\|_1.
\]
\end{lemma}

\begin{proof}
Write $g(\mathbf m,S):=\mathcal L(S,\hat a(\mathbf m))$.
For each action $S\subseteq[n]$ with $|S|\le k$, define
$\bar{S}:=[n]\setminus S$ and
\[
    z_S(e^{(t)}) := e^{(t)}_S + \Sigma_{SS}^{-1}\Sigma_{S\bar{S}}\,e^{(t)}_{\bar{S}}.
\]
Here, $e_S^{(t)} = (a_{i}-\hat{a}_{i}^{(t)})_{i \in S}.$
Because $z_{S}^{(t)}$ is linear in $\mathbf{e}^{(t)}$, there exists a matrix $B_S$ such that $z_S(e^{(t)}) = B_S e^{(t)}.$
From Proposition~\ref{prop:one-step-loss-corr}, we have that 
\[
    g(\mathbf{m}, S)
    =
    \sigma_\varepsilon^2
    +
    a_{\bar{S}}^\top \Sigma_{\bar{S}\mid S} a_{\bar{S}}
    + e(\mathbf{m})^\top M_S\, e(\mathbf{m}),
    \qquad
    M_S := B_S^\top \Sigma_{SS} B_S.
\]
Since \(|a_i|\le B\) and \(|\hat{a}_i(\cdot)|\le B\), every coordinate of
\(e(\mathbf{m})\) and \(e(\mathbf{m}')\) lies in \([-2B,2B]\), so
\(\|e(\mathbf{m})+e(\mathbf{m'})\|_\infty \le 4B\). Therefore

\begin{align*}
    |g(\mathbf{m},S) - g(\mathbf{m}',S)|
    &= \left|e(\mathbf{m})^\top M_S e(\mathbf{m}) - {e(\mathbf{m}')}^\top M_S e(\mathbf{m}')\right| \\
    &= \left|(e(\mathbf{m})-e(\mathbf{m}'))^\top M_S (e(\mathbf{m})+e(\mathbf{m}'))\right| \\
    &\leq \|e(\mathbf{m})-e(\mathbf{m}')\|_1\,\|M_S(e(\mathbf{m})+e(\mathbf{m}'))\|_\infty \\
    &\leq \|e(\mathbf{m})-e(\mathbf{m}')\|_1\,\|M_S\|_\infty\,\|e(\mathbf{m})+e(\mathbf{m}')\|_\infty \\
    &\leq 4B\|M_S\|_\infty\|e(\mathbf{m})-e(\mathbf{m}')\|_1.
\end{align*}
Since $e(\mathbf{m}) - e(\mathbf{m}') = -(\hat{a}(\mathbf{m}) - \hat{a}(\mathbf{m}'))$, and letting $L_{\Sigma,B}:=4B\max_{\substack{S\subseteq[n]\\ |S|\le k}}\|M_S\|_\infty$, we get that
\[
|g(\mathbf{m},S)-g(\mathbf{m}',S)| \leq L_{\Sigma,B} \lVert \hat a(\mathbf{m})-\hat a(\mathbf{m}') \rVert_{1}.
\]
\end{proof}

\begin{definition}[Recurrent set and frozen tail counts]
\label{def:recurrent-set}
Fix a feasible policy $\pi=(S_t)_{t\ge0}$, and let $\mathbf m^{(t)}$ be its induced
count-state trajectory. Define
\[
    R(\pi) := \{i\in[n] : i\in S_t \text{ for infinitely many } t\},
    \qquad
    N(\pi) := [n]\setminus R(\pi).
\]
Then every $j\in N(\pi)$ is selected only finitely many times, so there exist
$T_R\in\mathbb{N}$ and a vector $c_N^\pi\in\mathbb{N}^{N(\pi)}$ such that
\[
    \bigl(\mathbf m^{(t)}\bigr)_{N(\pi)} = c_N^\pi
    \qquad\text{for all }t\ge T_R.
\]
We refer to $R(\pi)$ as the recurrent set and to $c_N^\pi$ as the frozen
tail counts.
\end{definition}

\begin{proposition}[Asymptotic tail loss along a fixed policy]
\label{prop:asymptotic-tail-loss}
In the Gaussian squared-loss model, suppose \(\mathcal X=\mathcal N(0,\Sigma)\) with \(\Sigma\succ0\).
Fix a feasible policy $\pi=(S_t)_{t\ge0}$, let $R:=R(\pi)$, let
$N:=[n]\setminus R$, and let $c_N:=c_N^\pi$ be as in
Definition~\ref{def:recurrent-set}. For every subset $S\subseteq R$ with
$|S|\le k$, the limit
\[
    \bar{g}_\pi(S) := \lim_{t\to\infty} g(\mathbf m^{(t)}, S)
\]
exists and is given by
\[
    \bar{g}_\pi(S)
    =
    \sigma_\varepsilon^2
    +
    a_{\bar{S}}^\top \Sigma_{\bar{S}\mid S} a_{\bar{S}}
    +
    e_N(c_N)^\top \Sigma_{NS}\Sigma_{SS}^{-1}\Sigma_{SN} e_N(c_N),
    \qquad
    \bar{S} := [n]\setminus S,
\]
where
\[
    e_N(c_N) := \bigl(a_j - \hat{a}_j(c_j)\bigr)_{j\in N}.
\]
Moreover, if
\[
    \varepsilon_t
    :=
    \max_{\substack{S\subseteq R\\ |S|\le k}}
    \bigl|g(\mathbf m^{(t)}, S) - \bar{g}_\pi(S)\bigr|,
\]
then $\varepsilon_t\to0$.
\end{proposition}

\begin{proof}
Fix $S\subseteq R$ with $|S|\le k$. Since every $i\in R$ is selected infinitely
often under $\pi$, we have $(\mathbf m^{(t)})_i\to\infty$, and hence
\[
    \hat{a}_i\bigl((\mathbf m^{(t)})_i\bigr)\to a_i
    \qquad\text{for every }i\in R.
\]
Writing $e^{(t)} := e(\mathbf m^{(t)}) = a - \hat{a}(\mathbf m^{(t)})$, we have
\[
    e^{(t)}_i \to 0
    \qquad\text{for every }i\in R.
\]
On the other hand, by definition of $N$, the count vector on $N$ freezes:
there exists $c_N$ such that $(\mathbf m^{(t)})_N = c_N$ for all large enough $t$.
Hence
\[
    e^{(t)}_N = e_N(c_N)
    \qquad\text{for all large enough }t.
\]
Apply Proposition~\ref{prop:one-step-loss-corr} with state $\mathbf m^{(t)}$ and
action $S$:
\[
\begin{aligned}
    g(\mathbf m^{(t)}, S)
    =\,&
    \sigma_\varepsilon^2\\
    &+
    a_{\bar{S}}^\top \Sigma_{\bar{S}\mid S} a_{\bar{S}}\\
    &+
    \Bigl(
    e^{(t)}_S + \Sigma_{SS}^{-1}\Sigma_{S\bar{S}}\,e^{(t)}_{\bar{S}}
    \Bigr)^\top
    \Sigma_{SS}
    \Bigl(
    e^{(t)}_S + \Sigma_{SS}^{-1}\Sigma_{S\bar{S}}\,e^{(t)}_{\bar{S}}
    \Bigr).
\end{aligned}
\]
Decompose $\bar{S} = (R\setminus S)\cup N$. Since
\[
    e^{(t)}_S \to 0,
    \qquad
    e^{(t)}_{R\setminus S} \to 0,
    \qquad
    e^{(t)}_N = e_N(c_N)
\]
for all large enough $t$, the quadratic term converges to
\[
    \bigl(\Sigma_{SS}^{-1}\Sigma_{SN} e_N(c_N)\bigr)^\top
    \Sigma_{SS}
    \bigl(\Sigma_{SS}^{-1}\Sigma_{SN} e_N(c_N)\bigr)
    =
    e_N(c_N)^\top \Sigma_{NS}\Sigma_{SS}^{-1}\Sigma_{SN} e_N(c_N).
\]
This proves the displayed formula for $\bar{g}_\pi(S)$.

Since the action set $\{S\subseteq R : |S|\le k\}$ is finite, pointwise
convergence in $S$ is automatically uniform over that finite set. Hence
\[
    \varepsilon_t
    =
    \max_{\substack{S\subseteq R\\ |S|\le k}}
    \bigl|g(\mathbf m^{(t)}, S) - \bar{g}_\pi(S)\bigr|
    \to 0.
\]
\end{proof}

\begin{corollary}
\label{cor:generic-eventual-stationarity-regular}
Let $\mathcal{X} = \mathcal{N}(0,\Sigma)$. 
Assume that $\mathcal{L}$ satisfies Assumption~\ref{ass:regular-loss} and the limiting tail losses are analytically nondegenerate: for every recurrent set \(R\subseteq[n]\), writing \(N=[n]\setminus R\), every frozen count vector \(c_N\in\mathbb N_0^{|N|}\), and every two distinct feasible tail actions \(S,T\subseteq R\), the function
\[
    \mathbf a
    \mapsto
    \mathcal L(S,\bar{a}(\mathbf{a},c_N))
    -
    \mathcal L(T,\bar{a}(\mathbf{a},c_N))
\]
is real analytic and not identically zero.

If \(a\in\mathbb R^n\) is drawn from an absolutely continuous distribution,
then with probability one every infinite-horizon optimal policy is
eventually stationary: for each optimal policy
$\pi^\star$, there exists a \(T_d<\infty\) and
\(S^\star\subseteq[n]\), \(|S^\star|\le k\), such that
\[
    S_t=S^\star
    \qquad
    \text{for all }t\ge T_d .
\]
\end{corollary}

\begin{proof}[Proof of Corollary~\ref{cor:generic-eventual-stationarity-regular}]
Fix a recurrent set \(R\subseteq[n]\), write \(N=[n]\setminus R\), fix a
frozen count vector \(c_N\in\mathbb N_0^N\), and fix two distinct feasible tail
actions \(S,T\subseteq R\) with \(|S|,|T|\le k\). By analytic nondegeneracy, the
function
\[
    F_{R,c_N,S,T}(a)
    :=
    \bar{\mathcal L}_{R,c_N}(S;a)
    -
    \bar{\mathcal L}_{R,c_N}(T;a)
\]
is real analytic and not identically zero. A nonzero real-analytic function has
a Lebesgue-measure-zero zero set, so
\[
    E_{R,c_N,S,T}
    :=
    \{a\in\mathbb R^n:F_{R,c_N,S,T}(a)=0\}
\]
has Lebesgue measure zero. There are finitely many choices of \(R,S,T\) and
countably many choices of \(c_N\), so the union of all such tie sets is still
Lebesgue measure zero.

Since \(a\) is drawn from an absolutely continuous distribution, with
probability one no limiting tail tie occurs for any tuple \((R,c_N,S,T)\). On
this probability-one event, fix any infinite-horizon optimal policy
\(\pi^\star=(S_t)_{t\ge0}\). Let \(R=R(\pi^\star)\) and let \(c_N\) be its
frozen count vector on \(N=[n]\setminus R\). The finite set of feasible tail
actions \(\{S\subseteq R: |S|\le k\}\) has no pairwise ties in limiting tail
loss, and therefore has a unique minimizer. Theorem~\ref{thm:eventually-static-suffix-regular}
then implies that \(\pi^\star\) is eventually stationary. Because the
probability-one event is independent of the particular optimal policy, the
conclusion holds for every infinite-horizon optimal policy simultaneously.
\end{proof}

\paragraph{Gaussian squared-loss verification of analytic nondegeneracy.}
For completeness, we verify the analytic nondegeneracy condition in
Corollary~\ref{cor:generic-eventual-stationarity-regular} for the Gaussian
squared-loss limiting tail losses under any fixed-sign featurewise learning rule
\(e_j(c_j)=\lambda_j(c_j)(a_j-\hat a_j^{(0)})\), which includes the geometric
rules used in the main examples. Fix \(R\subseteq[n]\), write
\(N=[n]\setminus R\), and fix a frozen count vector \(c\in\mathbb N_0^N\).
For two distinct tail actions \(S,T\subseteq R\) with \(|S|,|T|\le k\),
Proposition~\ref{prop:asymptotic-tail-loss} implies that the difference between
their policy-specific limiting losses is
\[
\begin{aligned}
    \Delta_{R,c,S,T}(a)
    :=
    &\ a_{\bar S}^{\top}\Sigma_{\bar S\mid S}a_{\bar S}
    - a_{\bar T}^{\top}\Sigma_{\bar T\mid T}a_{\bar T} \\
    &+
    z_N(c)^\top
    \left(
        \Sigma_{NS}\Sigma_{SS}^{-1}\Sigma_{SN}
        -
        \Sigma_{NT}\Sigma_{TT}^{-1}\Sigma_{TN}
    \right)
    z_N(c),
\end{aligned}
\]
where
\[
    z_N(c):=
    \bigl(\lambda_j(c_j)(a_j-\hat a_j^{(0)})\bigr)_{j\in N}.
\]
Terms involving an empty selected set are interpreted by the usual block-matrix
convention. The common noise term \(\sigma_\varepsilon^2\) cancels in the
difference. For fixed \(R,c,S,T\), \(\Delta_{R,c,S,T}(a)\) is a polynomial in
\(a\).

This polynomial is not identically zero. If necessary, swap \(S\) and \(T\) so
that there exists \(j\in T\setminus S\). Set \(a_j=u\) and set all other
coordinates of \(a\) to zero. Since \(j\in R\), the frozen-error term
\(z_N(c)\) is independent of \(u\). The first fully learned term contains
\[
    u^2\operatorname{Var}(x_j\mid x_S),
\]
whose coefficient is strictly positive because \(\Sigma\succ0\), whereas the
second fully learned term has no \(u^2\) contribution because \(j\in T\).
Hence \(\Delta_{R,c,S,T}\) is a nonzero polynomial.

The zero set of a nonzero polynomial has Lebesgue measure zero. There are
finitely many choices of \(R,S,T\) and countably many frozen count vectors
\(c\), so the union of all tail-tie sets
\[
    \{\Delta_{R,c,S,T}(a)=0\}
\]
has Lebesgue measure zero.

\subsection{Computational Hardness}
\label{app:hardness-proof}

\begin{proof}[Proof of Theorem~\ref{thm:np_hard}]

We reduce from \textsc{Vertex Cover} on cubic graphs, which is an APX-complete
problem~\citep{hardness_cubic_graph}, then we will show that no FPTAS exists.
Let $G=(V,E)$ be a cubic graph on $n$ vertices, let $M$ be its adjacency
matrix, and let $k$ be the target cover size. Construct the one-shot instance
\[
    a := \mathbf{1}, \qquad
    \hat a := a,\qquad
    \sigma_\varepsilon^2:=1,\qquad
    \rho := \frac{1}{20n}, \qquad
    \Sigma := I+\rho M, \qquad
    B_0 := n-k+\frac{\rho}{2},\qquad
    B:=1+B_0.
\]

The equality \(\hat a=a\) sets the belief-error term to zero.
For this instance, the noisy one-shot objective is
\[
    F(S)=1+H(S),
    \qquad
    H(S):=a_{\bar{S}}^\top\, \Sigma_{\bar{S} \mid S}\, a_{\bar{S}}.
\]
We will show that deciding whether there exists a vertex cover of size $k$ on
$G=(V,E)$ is equivalent to deciding whether there exists an
$S \subseteq [n], |S|\le k$ with $F(S) \leq B.$

We first verify that $\Sigma$ is a valid positive definite correlation matrix.
Since $G$ is cubic, every eigenvalue of $M$ lies in $[-3,3]$. Hence every
eigenvalue of $\Sigma$ is at least $1-3\rho>0$ by Weyl's inequality.
Also, $\Sigma_{ii}=1$ for all $i$ because $M_{ii} = 0$.
Therefore $\Sigma$ is a positive definite correlation matrix.

Next, consider the policy-dependent part $H(S)$.
Now fix any subset $S \subseteq V$, and let $\bar{S}:=V\setminus S$.
By block
expansion,
\[
    \Sigma_{\bar{S}\mid S}
    = I_{\bar{S}}+\rho M_{\bar{S}\bar{S}}
    -\rho^2 M_{\bar{S}S}(I_S+\rho M_{SS})^{-1}M_{S\bar{S}}.
\]
Therefore,
\[
    H(S)
    = \mathbf{1}_{\bar{S}}^\top \Sigma_{\bar{S}\mid S}\,\mathbf{1}_{\bar{S}}
    = |\bar{S}| + 2\rho\, e_G(\bar{S}) - r(S),
\]
where $e_G(\bar{S})$ denotes the number of edges with both endpoints in
$\bar{S}$, and
\[
    r(S) := \rho^2\,
    \mathbf{1}_{\bar{S}}^\top
    M_{\bar{S}S}(I_S+\rho M_{SS})^{-1}M_{S\bar{S}}
    \mathbf{1}_{\bar{S}}.
\]
We next bound $r(S)$.
Because $r(S)$ is in quadratic form, we note that $r(S) \geq 0.$
Since every eigenvalue of $M_{SS}$ lies in $[-3,3]$,
we have
\[
    I_S+\rho M_{SS} \succeq (1-3\rho)I_S,
\]
and hence
\[
    (I_S+\rho M_{SS})^{-1} \preceq \frac{1}{1-3\rho}I_S.
\]
Thus
\[
    r(S) \leq \frac{\rho^2}{1-3\rho}\,\|M_{S\bar{S}}\mathbf{1}_{\bar{S}}\|_2^2.
\]
Because $G$ is cubic, every entry of $M_{S\bar{S}}\mathbf{1}_{\bar{S}}$ is at
most $3$, so
\[
    \|M_{S\bar{S}}\mathbf{1}_{\bar{S}}\|_2^2 \leq 9|S| \leq 9n.
\]
Therefore
\[
    r(S) \leq \frac{9n\rho^2}{1-3\rho} < \frac{\rho}{2}.
\]
We claim that
\[
    \exists\, S\subseteq V,\ |S|\leq k,\ F(S)\leq B
    \quad\Longleftrightarrow\quad
    G \text{ has a vertex cover of size at most } k.
\]
For the forward implication, suppose $G$ has a vertex cover of size at most
$k$. By adding arbitrary vertices if necessary, we may assume it has a vertex
cover $S$ of size exactly $k$. Then $\bar{S}=V\setminus S$ is an independent
set, so $e_G(\bar{S})=0$. Hence
\[
    F(S) = 1+n-k-r(S) \leq 1+n-k < 1+n-k+\frac{\rho}{2} = B.
\]
For the reverse implication, suppose there exists $S\subseteq V$ with
$|S|\leq k$ and $F(S)\leq B$.
If $|S|\leq k-1$, then $|\bar{S}|\geq n-k+1$, so
\[
    F(S) \geq 1+|\bar{S}|-r(S) > 1+n-k+1-\frac{\rho}{2} > 1+n-k+\frac{\rho}{2} = B,
\]
a contradiction. Therefore $|S|=k$.
Now suppose $e_G(\bar{S})\geq 1$. Then
\[
    F(S)
    = 1+|\bar{S}|+2\rho\, e_G(\bar{S})-r(S)
    \geq 1+n-k+2\rho-\frac{\rho}{2}
    = 1+n-k+\frac{3\rho}{2} > B,
\]
again a contradiction. Hence $e_G(\bar{S})=0$, so $\bar{S}$ is an independent
set and $S$ is a vertex cover of size $k$.
This proves the claimed equivalence. Since the reduction is polynomial-time,
the decision version is NP-hard, and therefore the optimization problem is
NP-hard.

Next, note the following two facts:
\begin{itemize}
    \item If $G$ has a vertex cover of size at most $k$, then there exists a
    feasible set $S \subseteq V$ with $|S| \leq k$ such that
    \[
        F(S) \leq 1+n-k.
    \]
    Hence
    \[
        \mathrm{OPT} := \min\{F(S) : S \subseteq [n],\ |S| \leq k\} \leq 1+n-k.
    \]
    \item If $G$ has no vertex cover of size at most $k$, then every feasible
    set $S \subseteq V$ with $|S| \leq k$ satisfies
    \[
        F(S) > 1+n-k+\frac{\rho}{2} = B.
    \]
\end{itemize}
Now choose the approximation parameter
\[
    \varepsilon := \frac{\rho}{4(n+1)}.
\]
Since $1/\varepsilon = 4(n+1)/\rho=80n(n+1)$, this choice is polynomial in the input size,
so the running time of the FPTAS remains polynomial.
Consider first a YES instance. Then $\mathrm{OPT} \leq 1+n-k$, so the FPTAS
returns a feasible set $\widehat{S}$ with
\[
    F(\widehat{S}) \leq (1+\varepsilon)\,\mathrm{OPT}
    \leq (1+\varepsilon)(1+n-k).
\]
Because $1+n-k \leq n+1$, we obtain
\[
    (1+\varepsilon)(1+n-k)
    \leq 1+n-k+\varepsilon(n+1)
    = 1+n-k+\frac{\rho}{4}
    < 1+n-k+\frac{\rho}{2}
    = B.
\]
Hence on YES instances the output value satisfies
\[
    F(\widehat{S}) < B.
\]
Now consider a NO instance. By the gap above, every feasible set $S$ with
$|S| \leq k$ satisfies $F(S) > B$. In particular,
\[
    F(\widehat{S}) > B.
\]
Therefore, by running $\mathcal{A}$ with approximation parameter
$\varepsilon = \rho/(4(n+1))$ and checking whether the returned value is below or
above the threshold $B$, we can decide \textsc{Vertex Cover} on cubic graphs
in polynomial time. This contradicts $P \neq NP$.
Hence no FPTAS exists for the noisy one-shot problem unless $P = NP$.
\end{proof}

\subsection{Finite-Horizon Dynamic Programming}
\label{app:finite-dp-proof}

\begin{proof}[Proof of Proposition~\ref{prop:finite-horizon-dp}]
Because beliefs are determined by the count state $\mathbf m^{(t)}$, the pair
$(t, \mathbf{m}^{(t)})$ is a sufficient state variable. If action $S$ is chosen at
state $\mathbf{m}^{(t)}$, then the current loss is $\mathcal L(S,\hat a(\mathbf{m}^{(t)}))$ and the next state is
deterministically $\mathbf{m}^{(t)} + \mathbf{1}_S$, where $\mathbf{1}_{S}$ is 1 for all $i \in S$ and 0 otherwise.
We prove the forward recursion by induction on \(t\). At \(t=0\),
\(V_0(\mathbf 0)=0\) and all other states have value \(+\infty\), so
\(V_0\) is exactly the minimum loss among length-zero prefixes. Suppose the
claim holds at time \(t\). Any length-\(t+1\) action sequence ending at
\(\mathbf m'\) consists of a length-\(t\) prefix ending at some state
\(\mathbf m\), followed by an action \(S\) satisfying
\(\mathbf m'=\mathbf m+\mathbf 1_S\). Its loss is the prefix loss plus
\(\delta^t\mathcal L(S,\hat a(\mathbf m))\). Minimizing over all such predecessors gives
exactly the update rule in Algorithm~\ref{algo:dp}. Conversely, every update in
the algorithm corresponds to a valid prefix followed by one valid action.
Thus \(V_t(\mathbf m)\) equals the best prefix loss for every layer \(t\), and
\(\min_{\mathbf m}V_T(\mathbf m)\) is the optimal finite-horizon value from the
initial state. The stored parent pointers record a predecessor attaining each
minimum, so backtracking from any minimizing terminal state recovers an optimal
action sequence.

There are at most $(T+1)^n = O(T^n)$ possible count states, and each state
considers \(\sum_{j=0}^k\binom{n}{j}\) actions. Filling the layered table forward therefore
takes \(O(T^n\sum_{j=0}^k\binom{n}{j})\) time and $O(T^n)$ memory.
\end{proof}

\subsection{Truncation Approximation}
\label{app:truncation-proof}

\begin{proof}[Proof of Proposition~\ref{prop:cutoff-additive}]
For every policy $\pi$,
\[
    0
    \le
    J(\delta,\pi) - J_{\bar{T}}(\delta,\pi)
    =
    \sum_{t=\bar{T}}^\infty \delta^t \mathcal L(S_t,\hat a(\mathbf m^{(t)}))
    \le
    \sum_{t=\bar{T}}^\infty \delta^t \bar{L}
    =
    \frac{\bar{L}\,\delta^{\bar{T}}}{1-\delta}.
\]
Therefore
\[
\begin{aligned}
    J(\delta,\pi^{(\bar{T})})
    &\le J_{\bar{T}}(\delta,\pi^{(\bar{T})}) + \frac{\bar{L}\,\delta^{\bar{T}}}{1-\delta}\\
    &\le J_{\bar{T}}(\delta,\pi^\star) + \frac{\bar{L}\,\delta^{\bar{T}}}{1-\delta}\\
    &\le J(\delta,\pi^\star) + \frac{\bar{L}\,\delta^{\bar{T}}}{1-\delta},
\end{aligned}
\]
where the middle inequality uses the optimality of the first $\bar{T}$ actions of
$\pi^{(\bar{T})}$ for the truncated problem.
\end{proof}

\begin{remark}[Tightness of the suffix-agnostic tail bound]
\label{rem:cutoff-tail-tightness}
The dependence $\bar L\delta^{\bar T}/(1-\delta)$ in Proposition~\ref{prop:cutoff-additive} cannot be
improved for a guarantee that allows an arbitrary completion after the
finite-horizon prefix. Consider the Gaussian squared-loss instance with
$n=2$, $k=1$, $\Sigma=I$, zero noise, $a=(\sqrt{\bar L},0)$, and
$\hat a^{(0)}=a$. Then selecting $S=\{1\}$ incurs loss $0$ in every round,
whereas selecting $S=\{2\}$ incurs loss $\bar L$ in every round. For any
truncation length $\bar T$, the finite-horizon optimum selects $\{1\}$ in all
first $\bar T$ rounds. However, an admissible arbitrary completion may select
$\{2\}$ forever afterwards. Relative to the infinite-horizon optimal policy
that always selects $\{1\}$, this completion has excess loss
\[
\sum_{t=\bar T}^{\infty}\delta^t \bar L
=
\frac{\bar L\delta^{\bar T}}{1-\delta}.
\]
Thus Proposition~\ref{prop:cutoff-additive} is tight as a suffix-agnostic statement. This example
does not rule out sharper instance-dependent guarantees for a specified
stationary-completion rule.
\end{remark}

%% file: sections_arxiv/arxiv_experiments.tex
\section{Model-Variant Experiments}
\label{sec:model-variant-experiments}

\newcommand{\arxivmissingfigure}[1]{%
    \fbox{\parbox[c][1.1in][c]{0.92\linewidth}{\centering\scriptsize Run \texttt{\detokenize{#1}} to generate this figure.}}%
}

\label{sec:variant-x-loss-phi}
The main experiments use Gaussian covariates, squared loss, and geometric learning curves.
We add a four-scenario variant suite that changes one modeling component at a time while keeping \(n=2\), \(k=1\), observation-noise variance \(10^{-3}\), planning discount \(\delta=0.99\), and three random draws of \(a,\hat a^{(0)}\sim\mathrm{Uniform}([0,1]^2)\).
The curve labels in Figure~\ref{fig:variant-x-loss-phi} refer to the following settings:
\begin{enumerate}
    \item \textbf{Baseline}: \(x\sim\mathcal N(0,\Sigma_\rho)\), squared loss, and \(\phi_i(m)=1.05^{-2m}\).
    \item \textbf{Beta distribution}: \(x\) has standardized Beta\((2,5)\) marginals, squared loss, and \(\phi_i(m)=1.05^{-2m}\).
    \item \textbf{Huber loss}: \(x\sim\mathcal N(0,\Sigma_\rho)\), Huber loss with threshold \(1\), and \(\phi_i(m)=1.05^{-2m}\).
    \item \textbf{Power-law \(\phi\)}: \(x\sim\mathcal N(0,\Sigma_\rho)\), squared loss, and \(\phi_i(m)=(1/(m+1))^{0.75}\).
\end{enumerate}
Here \(\Sigma_\rho=(1-\rho)I+\rho\mathbf 1\mathbf 1^\top\).
For all four scenarios, the one-step expected loss is evaluated by the same fixed empirical oracle.
For each scenario, correlation setting, and random repeat, we draw \(N_{\mathrm{MC}}=1024\) samples \(\{(x^{(r)},\varepsilon^{(r)})\}_{r=1}^{N_{\mathrm{MC}}}\) and reuse them for every DP state-action loss:
\[
    \widehat L(m,j)
    =\frac{1}{N_{\mathrm{MC}}}\sum_{r=1}^{N_{\mathrm{MC}}}
    \ell\!\left(\widehat y_{m,j}^{(r)}-y^{(r)}\right),
    \qquad
    y^{(r)}=a^\top x^{(r)}+\varepsilon^{(r)} .
\]
Here \(m\) is the current vector of observation counts, \(j\) is the coordinate selected by the action, and \(\widehat y_{m,j}^{(r)}\) is the prediction after observing coordinate \(j\) and imputing the hidden coordinate.
Thus \(N_{\mathrm{MC}}=1024\) is only the Monte Carlo resolution used to approximate the expected loss; it is not a distributional parameter.
We use this empirical oracle even for the Gaussian squared-loss baseline so that the four curves differ only in the model variant, not in the loss-evaluation method.

The Beta-distribution scenario requires one extra specification: how to draw \(x\) with Beta\((2,5)\) marginals while still using \(\rho\) to control dependence between the two coordinates.
For each sample \(r\), draw
\[
    z^{(r)}\sim \mathcal N(0,\Sigma_\rho),
    \qquad
    u_i^{(r)}=\Phi\!\left(z_i^{(r)}\right),
    \qquad
    \tilde x_i^{(r)}=F^{-1}_{\mathrm{Beta}(2,5)}\!\left(u_i^{(r)}\right),
\]
where \(\Phi\) is the standard normal CDF and \(F^{-1}_{\mathrm{Beta}(2,5)}\) is the quantile function of a Beta\((2,5)\) random variable.
We then standardize each coordinate in the Monte Carlo sample,
\[
    \bar x_i=\frac{1}{N_{\mathrm{MC}}}\sum_{s=1}^{N_{\mathrm{MC}}}\tilde x_i^{(s)},
    \qquad
    s_i^2=\frac{1}{N_{\mathrm{MC}}}\sum_{s=1}^{N_{\mathrm{MC}}}\left(\tilde x_i^{(s)}-\bar x_i\right)^2,
    \qquad
    x_i^{(r)}=\frac{\tilde x_i^{(r)}-\bar x_i}{s_i}.
\]
The Gaussian draw \(z^{(r)}\) is used only to induce dependence between the two non-Gaussian coordinates; in this scenario, the plotted \(\rho\) is the latent-Gaussian correlation parameter before the Beta marginal transform.
In the Huber-loss scenario, the squared residual loss is replaced by
\[
    \ell_\tau(r)=
    \begin{cases}
        \frac{1}{2}r^2, & |r|\le \tau,\\
        \tau\bigl(|r|-\frac{1}{2}\tau\bigr), & |r|>\tau,
    \end{cases}
\]
with residual \(r=\hat y-y\) and threshold \(\tau=1\).
The conditional imputation in \(\widehat y_{m,j}^{(r)}\) is estimated by least squares: linear imputation for the Gaussian scenarios and cubic imputation for the Beta-distribution scenario.
Therefore Figure~\ref{fig:variant-x-loss-phi} does not rely on the Gaussian squared-loss closed form.

Figure~\ref{fig:variant-x-loss-phi} plots the resulting policies under these four scenarios.
The top panel reproduces the empirical part of Figure~\ref{fig:static-suffix-thresholds} by plotting the recovered dynamic phase length \(T_d\) against correlation; the bottom panel reproduces the Figure~\ref{fig:truncated-dp-runtime-quality} protocol by evaluating the truncated-DP policy obtained at horizon \(\bar T\), appending its recovered stationary suffix, and reporting retained optimality and runtime.
The trends are consistent with the main experiments: the dynamic phase remains short under weak or moderate dependence and grows primarily when the two covariates are strongly dependent, while truncated DP retains nearly all of the reference value once \(\bar T\) is large enough.
This supports the qualitative robustness of our conclusions to the covariate distribution, the loss function, and the learning-curve shape in the tested variants.
The theorem certificate from Figure~\ref{fig:static-suffix-thresholds} is not plotted for the non-Gaussian or non-squared-loss variants because the current certificate is specialized to the analytic Gaussian squared-loss oracle.

\begin{figure}[t]
    \centering
    \IfFileExists{sections_arxiv/figures/variant_static_suffix_thresholds.pdf}{%
        \includegraphics[width=\linewidth]{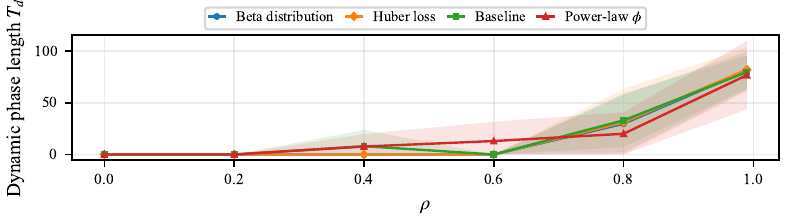}%
    }{%
        \arxivmissingfigure{bash/figures/run_variant_figures.sh}%
    }
    \vspace{0.35em}
    \IfFileExists{sections_arxiv/figures/variant_truncated_dp_runtime_quality.pdf}{%
        \includegraphics[width=\linewidth]{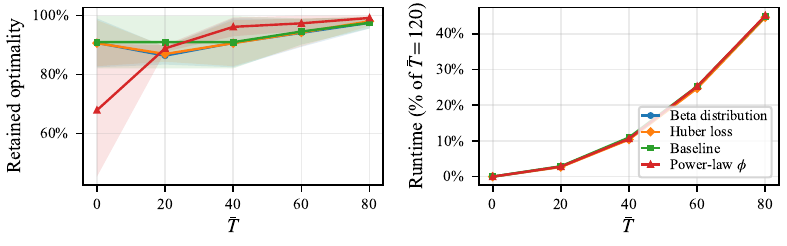}%
    }{%
        \arxivmissingfigure{bash/figures/run_variant_figures.sh}%
    }
    \caption{Four-scenario model-variant suite that changes one modeling component at a time while keeping \(n=2,k=1\), \(\delta=0.99\), and three random draws of \((a,\hat a^{(0)})\). The scenarios vary the feature distribution, loss function, or human learning curve relative to the Gaussian squared-loss geometric-learning baseline. Top: empirical dynamic phase length \(T_d\) as the correlation parameter \(\rho\) varies. Bottom: retained optimality over evaluation horizon \(T=180\) and runtime normalized by the \(\bar T=120\) solve as the truncation horizon \(\bar T\) varies at \(\rho=0.8\). Bands show mean \(\pm 1.96\) standard errors. Across variants, \(T_d\) grows mainly under strong correlation, and the largest plotted \(\bar T\) retains at least \(97\%\) of the reference policy value.}
    \label{fig:variant-x-loss-phi}
\end{figure}